\documentclass{article}
\usepackage[utf8]{inputenc}
\usepackage{geometry}
\usepackage{amsmath,bm,amsthm}
\usepackage[colorlinks=true,allcolors=blue]{hyperref}
\usepackage{bbm}

\usepackage[style=ieee,backend=biber]{biblatex}
\usepackage{amssymb,amsfonts}
\usepackage[ruled,vlined]{algorithm2e}
\usepackage{algorithmic}
\usepackage{graphicx}
\usepackage{textcomp}
\usepackage{enumitem}
\usepackage{xcolor}
\usepackage{float}
\usepackage{appendix}
\usepackage{subcaption}
\usepackage{caption}
\usepackage{wrapfig}
\usepackage{multirow}
\usepackage{graphicx}
\usepackage{wrapfig}

\def\BibTeX{{\rm B\kern-.05em{\sc i\kern-.025em b}\kern-.08em
    T\kern-.1667em\lower.7ex\hbox{E}\kern-.125emX}}

\newcommand\numberthis{\addtocounter{equation}{1}\tag{\theequation}}
\newtheorem{theorem}{Theorem}
\newtheorem{lemma}{Lemma}
\newtheorem{definition}{Definition}
\newtheorem{corollary}{Corollary}
\newtheorem{proposition}{Proposition}

\newtheorem{remark}{Remark}

\DeclareMathOperator{\argmin}{arg min}

\DeclareMathOperator{\Var}{Var}

\title{Implicit Regularization Properties of Variance Reduced Stochastic Mirror Descent}
\author{Yiling Luo, Xiaoming Huo, Yajun Mei}
\date{April 2022}

\begin{document}

\maketitle

\begin{abstract}
    In machine learning and statistical data analysis, we often run into objective function that is a summation: the number of terms in the summation possibly is equal to the sample size, which can be enormous.
In such a setting, the stochastic mirror descent (SMD) algorithm is a numerically efficient method---each iteration involving a very small subset of the data. 
The variance reduction version of SMD (VRSMD) can further improve SMD by inducing faster convergence. 
On the other hand, algorithms such as gradient descent and stochastic gradient descent have the implicit regularization property that leads to better performance in terms of the generalization errors.
Little is known on whether such a property holds for VRSMD. 
We prove here that the discrete VRSMD estimator sequence converges to the minimum mirror interpolant in the linear regression. 
This establishes the implicit regularization property for VRSMD.
As an application of the above result, we derive a model estimation accuracy result in the setting when the true model is sparse. 
We use numerical examples to illustrate the empirical power of VRSMD. 
\end{abstract}

\section{Introduction}
In statistics and machine learning, it is common to optimize an objective function that is a finite-sum. 
SMD efficiently optimizes such an objective by using a subset of data to do one step update of the variable/parameter. 
Further adopting the variance reduction technique to SMD, we get the VRSMD algorithm that enjoys fast convergence \cite{li2021variance,allen2017katyusha}. 

The implicit regularization is a relatively new concept \cite{bartlett2021deep} that explains why a result of an algorithm generalizes well in some overparameterized models \cite{bartlett2021deep,NEURIPS2020_37740d59}.
It refers to the fact that an algorithm can automatically select a minimum norm solution, which is not explicitly induced by the objective function. 
There are works on implicit regularization for Gradient Descent \cite{NEURIPS2019_5cf21ce3,zhao2019implicit,fan2021understanding,NEURIPS2019_c0c783b5}, Stochastic Gradient Descent \cite{ali2020implicit,azulay2021implicit,smith2021on,pmlr-v125-blanc20a}, and Stochastic Mirror Descent \cite{azizan2018stochastic}. 
Considering the computational advantage of VRSMD compared to all the algorithms above, it would be even better if VRSMD also has the useful implicit regularization property.

From technical point of view, our work contains the following two results:
\begin{itemize}
    \item In linear regression (including underfitting and overfitting), we show that the solution sequence of VRSMD converges to the minimum mirror interpolant, which is the implicit regularization property of VRSMD, and we also specify the convergence rate. 
    \item In sparse regression, by choosing a proper mirror map, we show that the implicit regularization estimator finds the sparse true parameter with a small error. Moreover, compared with the deterministic algorithms in \cite{NEURIPS2019_5cf21ce3,zhao2019implicit}, our algorithm is equally good in estimating a sparse truth while being computationally faster, as supported by our experiments. 
\end{itemize}

From the application point of view, our result shows that 
the Mirror Descent and its variants are useful to explore the low dimensional geometric structure from high dimensional data, which leads to nice generalization properties. 

\textbf{Notation}.
The following notations are used throughout this paper.
For a matrix $X\in \mathbbm{R}^{n\times p}$, we denote by $\mathrm{col}(X):= \{\mathbf{u}\in \mathbbm{R}^n : \exists \mathbf{v}\in \mathbbm{R}^p, \mathbf{u} = X \mathbf{v}\}$ the column space of $X$, and we denote by $\mathcal{N}(X):= \{\mathbf{v}\in \mathbbm{R}^p : X \mathbf{v} = \mathbf{0}\}$ the null space of $X$. For a vector $\mathbf{v}\in \mathbbm{R}^p$, we use the definition of $\ell_p$ norm of $\mathbf{v}$ that $\|\mathbf{v}\|_{p} = (\sum_{i} |v_i|^p)^{1/p}$ for $p\geq 1$ and we denote the number of non-zero elements in $\mathbf{v}$ as $\|\mathbf{v}\|_0$. For a subset of indexes $I\subset \{1,\ldots,p\}$, we define $\mathbf{v}_I := (v_i)_{i\in I}$, and denote the cardinality of $I$ as $|I|$. For a set $\mathcal{X} \subset \mathbbm{R}^p$, define $P_{\mathcal{X}}\mathbf{v} = \argmin_{\mathbf{u}\in \mathcal{X}} \|\mathbf{u} - \mathbf{v}\|_2$. For two non-negative-valued functions $a(x)$ and $b(x)$, we denote $a(x) \sim \mathcal{O}(b(x))$ if there exists an absolute constant $C$ such that $a(x) \leq Cb(x)$; and we denote $a(x) \sim \Theta(b(x))$ if there are absolute constants $c, C$ such that $cb(x) \leq a \leq Cb(x)$.

\textbf{Organization}. The rest of the paper is as follows. 
In Section \ref{ch3:sec:02}, we describe our problem formulation and algorithm. 
Section \ref{ch3:sec:03} states the main theory on the implicit regularization. 
Section \ref{ch3:sec:04} develops insight into the implicit regularization and establishes the sparse recovery property. 
Section \ref{ch3:sec:05} supports the theory on implicit regularization by simulations and experiments.
In Section \ref{ch3:sec:06}, we discuss the finding of our work and some future directions. 
All proofs are deferred to the appendix.

\section{Formulation and Algorithm}\label{ch3:sec:02}
To better present the material, we split this section into two subsections. In Subsection \ref{ch3:sec:02.01} we formulate the optimization problem motivated from linear regression. In Subsection \ref{ch3:sec:02.02}, we present the Variance Reduction Stochastic Mirror Descent (VRSMD) algorithm for solving such an optimization problem.

\subsection{Formulation}\label{ch3:sec:02.01}
Assume we observe data pairs $\{(\mathbf{x}_i,y_i)\in \mathbbm{R}^p \times \mathbbm{R}\}_{i=1}^n$, the goal is to predict the response $y$ based on $\mathbf{x}$. Under the empirical risk minimization framework, we consider the general optimization problem of the form 
\begin{equation}\label{ch3:eq:sec2eq1}
\min_{\boldsymbol{\beta}} F(\boldsymbol{\beta}) = \frac{1}{n}\sum_{i=1}^n f_i(\boldsymbol{\beta}; (\mathbf{x}_i,y_i)),
\end{equation}
and we shorten $ f_i(\boldsymbol{\beta}; (\mathbf{x}_i,y_i))$ as $f_i(\boldsymbol{\beta})$ to simplify the notation.

As a concrete example, for the linear regression model, the classical least squares method is to find coefficient $\boldsymbol{\beta}$ that minimizes the objective function
\begin{equation}
\label{ch3:eq:sec2eq2}
        \min_{\boldsymbol{\beta}} F(\boldsymbol{\beta}) =\frac{1}{2n}\sum_{i=1}^n\left(\mathbf{x}_i^T\boldsymbol{\beta} - y_i\right)^2 =\frac{1}{2n}\|X\boldsymbol{\beta} - y\|_2^2,
\end{equation}
where we denote $X = [\mathbf{x}_1, \ldots , \mathbf{x}_n]^T\in \mathbbm{R}^{n\times p}$, and $\mathbf{y} = [y_1,\ldots,y_n]^T \in \mathbbm{R}^{n}$. 

When problem \eqref{ch3:eq:sec2eq2} has non-unique solutions, it is important yet nontrivial to find a solution that has nice generalization property. It is well known that  
when running Gradient Descent algorithm with initialization $\boldsymbol{\beta}_0 = \mathbf{0}$ on \eqref{ch3:eq:sec2eq2}, the corresponding solution is the minimal $\ell_2$ norm solution among all solutions of \eqref{ch3:eq:sec2eq2}, see \cite{wu2021direction}. Also, the properties of SMD are studied in \cite{azizan2018stochastic}. It is unknown what happens to the solution if we run variants of SMD, for example, variance reduced SMD.

\subsection{VRSMD Algorithm}\label{ch3:sec:02.02}
Let us now present the main idea of the variance reduced stochastic mirror descent (VRSMD) algorithm as follows.

To understand why we need variance reduction, consider the Stochastic Mirror Descent(SMD) algorithm using a strictly convex and differentiable mirror map $\psi(\cdot)$. At step $t$, the SMD updates $\boldsymbol{\beta}_{t+1}$ such that
\begin{equation*}
\nabla\psi(\boldsymbol{\beta}_{t+1}) = \nabla \psi(\boldsymbol{\beta}_{t}) - \eta_t \nabla f_{i_t}(\boldsymbol{\beta}_{t}),
\end{equation*}
where $i_t$ is randomly sampled from $\{1,\ldots,n\}$. Now the term $\nabla f_{i_t}(\boldsymbol{\beta}_{t})$ has $\mathbbm{E}[\nabla f_{i_t}(\boldsymbol{\beta}_{t})] = \nabla F(\boldsymbol{\beta}_{t})$, so SMD has unbiased update compared to Mirror Descent, where the update is $\nabla\psi(\boldsymbol{\beta}_{t+1}) = \nabla \psi(\boldsymbol{\beta}_{t}) - \eta_t \nabla F(\boldsymbol{\beta}_{t})$. However, in general $\Var[\nabla f_{i_t}(\boldsymbol{\beta}_{t})]\neq 0$ for any $\boldsymbol{\beta}_{t}$, so we need $\eta _t\to 0$, which 
may lead to slow convergence.

Variance reduction addresses the issues above by replacing $\nabla f_{i_t}(\boldsymbol{\beta}_{t})$ with term $A_{t}$ such that
\begin{itemize}
    \item $\mathbbm{E}[A_t]  = \mathbbm{E}[\nabla f_{i_t}(\boldsymbol{\beta}_{t})] $ to keep unbiased update;
    \item $\Var[A_t]  < \Var[\nabla f_{i_t}(\boldsymbol{\beta}_{t})] $ to control variance.
\end{itemize}
One choice of $A_t$ is $A_t = \nabla f_{i_t}(\boldsymbol{\beta}_{t}) - B_t + \mathbbm{E}[B_t]$, where $B_t$ and $\nabla f_{i_t}(\boldsymbol{\beta}_{t})$ are positively correlated with correlation coefficient $r > 0.5$ and $\Var[B_t]\approx \Var[\nabla f_{i_t}(\boldsymbol{\beta}_{t})]$. For this $A_t$ one can check that $\mathbbm{E}[A_t] = \mathbbm{E}[\nabla f_{i_t}(\boldsymbol{\beta}_{t})]$ and $\Var[A_t] = \Var[\nabla f_{i_t}(\boldsymbol{\beta}_{t}) - B_t] =\Var[\nabla f_{i_t}(\boldsymbol{\beta}_{t})] - 2r \sqrt{\Var[\nabla f_{i_t}(\boldsymbol{\beta}_{t})]\Var[B_t]} + \Var[B_t] <\Var[\nabla f_{i_t}(\boldsymbol{\beta}_{t})] $. For a proper $B_t$ such that $\Var(A_t) \stackrel{t\to\infty}{\longrightarrow} 0$, the algorithm converges for a fixed $\eta$. 

For illustration purpose, we use the variance reduction technique in \cite{johnson2013accelerating} to get the VRSMD Algorithm \ref{ch3:algo:SMD}. However, one should note that this framework applies to other variance reduction methods such as SARAH \cite{pmlr-v70-nguyen17b} and SPIDER \cite{NEURIPS2018_1543843a}. 

\begin{algorithm}[ht]
\textbf{Input}: An objective function $F(\cdot) = \frac{1}{n}\sum_{i=1}^n f_i(\cdot)$, and a strictly convex and differentiable mirror map $\psi(\cdot)$\;
\textbf{Initialization}: Initialize $\Tilde{\boldsymbol{\beta}}^{0}$. Choose the step-size $\eta$, outer iteration number $S$, inner iteration number $m$. Denote the estimator at $t$th inner iteration of $s$th outer iteration as $\boldsymbol{\beta}_{t}^{s}$. Set $\boldsymbol{\beta}_{1}^{1} = \boldsymbol{\beta}_{m+1}^{0} = \Tilde{\boldsymbol{\beta}}^{0}$\;
\For{Outer iteration s = 1,\ldots,S}{
Calculate $\nabla F(\Tilde{\boldsymbol{\beta}}^{s-1})$\;
\For{Inner iteration t = 1,\ldots,m}{
Randomly sample $i_t$ from $\{1,\ldots,n\}$, calculate
\begin{equation}
\label{ch3:eq:SMD_update}
\mathbf{v}_{t}^{s} =\nabla f_{i_t}(\boldsymbol{\beta}_{t}^{s}) - \nabla f_{i_t}(\Tilde{\boldsymbol{\beta}}^{s-1}) + \nabla F(\Tilde{\boldsymbol{\beta}}^{s-1}),
\end{equation}
and update $\boldsymbol{\beta}_{t+1}^{s}$ such that 
\begin{equation}
\nabla\psi(\boldsymbol{\beta}_{t+1}^{s}) = \nabla\psi(\boldsymbol{\beta}_{t}^{s}) - \eta \mathbf{v}_{t}^{s}.
\end{equation}
}
Set $\Tilde{\boldsymbol{\beta}}^{s}$ to be a uniform random sample from $\{\boldsymbol{\beta}_{1}^{s},\ldots,\boldsymbol{\beta}_{m}^{s}\}$\;
\textbf{Option I}: Set $\boldsymbol{\beta}_{1}^{s+1} = \boldsymbol{\beta}_{m+1}^{s}$\;
\textbf{Option II}: Set $\boldsymbol{\beta}_{1}^{s+1} = \Tilde{\boldsymbol{\beta}}^{s}$\;
}
\textbf{Option I}: Output $\boldsymbol{\beta}_a$ chosen uniformly random from $\{\{\boldsymbol{\beta}_t^{s}\}_{t=1}^{m}\}_{s = 1}^{S}$\;
\textbf{Option II}: Output $\boldsymbol{\beta}_a = \Tilde{\boldsymbol{\beta}}^{S}$.
\caption{Variance Reduced Stochastic Mirror Descent (VRSMD)
}
\label{ch3:algo:SMD}
\end{algorithm}
\begin{remark}
We note the complexity of VRSMD as follows: The total number of 
stochastic-first-order calls (i.e. SFO complexity) of Algorithm \ref{ch3:algo:SMD} is $\mathcal{O}(nS+mS)$. A popular choice of the inner loop number $m$ is $\Theta(n)$, which leads to $\mathcal{O}(1)$ SFO complexity per inner loop. 
\end{remark}

The variance reduction component of VRSMD is $\mathbf{v}_{t}^{s}$ in \eqref{ch3:eq:SMD_update}. Note that it is equivalent to the variance reduction scheme we talked above by taking $B_t = \nabla f_{i_t}(\Tilde{\boldsymbol{\beta}}^{s-1})$. 
The conditions we list there on $B_t$ hold when $\boldsymbol{\beta}_{t}^{s}$ and $\Tilde{\boldsymbol{\beta}}^{s-1}$ are close, which happens by taking moderate values of the inner iteration number $m$ and the step-size $\eta$.

We show that the VRSMD algorithm is a generalization of the SVRG algorithm in \cite{johnson2013accelerating,reddi2016stochastic}: For the special case of $\psi(\cdot) = \frac{1}{2} \|\cdot\|_2^2$, we have $\nabla\psi(\boldsymbol{\beta}) =\boldsymbol{\beta}$, then \eqref{ch3:eq:SMD_update} updates $  \boldsymbol{\beta}_{t+1}^{s} $ as:
\begin{equation*}
    \boldsymbol{\beta}_{t+1}^{s} = \boldsymbol{\beta}_{t}^{s} - \eta \mathbf{v}_{t}^{s},
\end{equation*}
which is the SVRG update, and VRSMD reduces to SVRG.

\section{Implicit Regularization}\label{ch3:sec:03}
In this section, we present the implicit regularization property of the VRSMD solution. To do so, it is necessary to first show that the VRSMD converges.  

To begin with, we introduce some definitions that will be useful in our theory. 
\begin{definition}[L-smoothness]\label{ch3:def:smooth}
$f$ is $L$-smooth with respect to $\|\cdot\|$ norm if there exists a constant $L>0$ such that 
$$\|\nabla f(\mathbf{u}) - \nabla f(\mathbf{w})\|_* \leq L\|\mathbf{u} - \mathbf{w}\|, \forall \mathbf{u},\mathbf{w},$$
where $\|\cdot\|_*:=\max_{y:\|y\|=1}\langle y,\cdot\rangle$ is the dual norm of $\|\cdot\|$.
\end{definition}
\begin{definition}[$\alpha$-strongly convex]\label{ch3:def:convex}
$f$ is $\alpha$-strongly convex with respect to $\|\cdot\|$ norm if there exists a constant $\alpha >0$ such that 
$$f(\mathbf{u})\geq f(\mathbf{w}) + \nabla f(\mathbf{w}) ^T (\mathbf{u} - \mathbf{w}) + \frac{\alpha}{2}\|\mathbf{u} - \mathbf{w}\|^2, \forall \mathbf{u},\mathbf{w}.$$
\end{definition}

\begin{definition}[Quadratic growth (QG)]\label{ch3:del:qg}
Let $\mathcal{X}$ be the set of all minimizers of $f$. 
$f$ satisfies QG condition w.r.t. $\|\cdot\|$ if
\begin{equation}\label{ch3:eq:QG}
    \frac{\mu}{2}\|\mathbf{u} - P_{\mathcal{X}}\mathbf{u}\|^2\leq  f(\mathbf{u}) - f(P_{\mathcal{X}}\mathbf{u}), \forall\mathbf{u}.
\end{equation}

\end{definition}


\begin{definition}[$\epsilon$-solution]\label{ch3:def:approx_opt}
For the optimization problem
\begin{align}
\begin{split}\label{ch3:eq:optmization_problem}
Opt = \min_{x\in B}\{ f(x): g_i(x)\leq 0, 1\leq i\leq m\}.
\end{split}
\end{align}
$x_{\epsilon}\in B$ is called an $\epsilon$-solution to \eqref{ch3:eq:optmization_problem} if
\begin{align*}
    \begin{split}
        &f(x_{\epsilon}) - Opt \leq \epsilon,\\
        &g_i(x_{\epsilon})\leq \epsilon, 1\leq i\leq m.
    \end{split}
\end{align*}
\end{definition}

\begin{definition}[Restricted eigenvalue (RE)]\label{ch3:def:RE}
$X$ satisfies $(s,\gamma)$-RE condition if for any $ \boldsymbol{\beta}$ such that $\|\boldsymbol{\beta}\|_0 \leq s$ we have

\begin{equation*}
    \frac{\frac{1}{n}\|X\boldsymbol{\beta}\|_2^2}{\|\boldsymbol{\beta}\|_2^2} \geq \gamma .
\end{equation*}
\end{definition}

\begin{definition}[$s$-good]\label{ch3:def:sgood}
A matrix $X_{n\times p}$ is $s$-good if $\exists \kappa < \frac{1}{2}$ such that $\forall \mathbf{u}\in \mathcal{N}(X)\subset \mathbbm{R}^p$ and $\forall I\subset\{1,\ldots,p\}$ with $|I|\leq s$, we have
\begin{center}
 $\|\mathbf{u}_{I}\|_1 \leq \kappa\|\mathbf{u}\|_1$.
\end{center}
\end{definition}

Next, let us present the convergence result of VRSMD:
\begin{proposition}\label{CH3:THM:02}
Assume $F(\cdot) = \frac{1}{n}\sum_{i=1}^n f_i(\cdot)$ has every $f_i(\cdot)$ convex and $L$-smooth w.r.t. an arbitrary norm $\|\cdot\|$, and $\psi(\cdot)$ is $\alpha$-strongly convex w.r.t. $\|\cdot\|$. Denote $\boldsymbol{\beta}^{*} = \arg\min F(\cdot)$. 

\textbf{(a)} Run Option I of Algorithm \ref{ch3:algo:SMD} on $F$ with $\eta <\frac{\alpha}{24L}$, then we have 
\begin{align}
    \begin{split}
    \label{ch3:eq:thm1eq1}
        \mathbbm{E}[F(\boldsymbol{\beta}_a) - F(\boldsymbol{\beta}^*)] \leq \frac{\alpha}{(\alpha\eta-24L\eta^2)T}\times\\
        \quad\left[D_{\psi}(\boldsymbol{\beta}^*,\Tilde{\boldsymbol{\beta}}^{0})+ \frac{12L\eta^2 m}{\alpha}(F(\Tilde{\boldsymbol{\beta}}^{0}) - F(\boldsymbol{\beta}^*))\right],
    \end{split}
\end{align}
where $T = m\cdot S$, and 
    $D_{\psi}(\boldsymbol{\beta}^*,\Tilde{\boldsymbol{\beta}}^{0}) := \psi(\boldsymbol{\beta}^*) - \psi(\Tilde{\boldsymbol{\beta}}^{0}) - \langle\nabla \psi(\Tilde{\boldsymbol{\beta}}^{0}), \boldsymbol{\beta}^*-\Tilde{\boldsymbol{\beta}}^{0}\rangle$ is the Bregman divergence.

\textbf{(b)} If we further assume that $F(\cdot)$ satisfies the QG condition in \eqref{ch3:eq:QG} with constant $\mu$, and that $\psi(\cdot)$ is $\ell$-smooth, all w.r.t. $\|\cdot\|$, and also suppose that we run Option II of Algorithm \ref{ch3:algo:SMD} with a large enough $m$ such that 
\begin{equation}\label{ch3:eq:thm1eq3}
\tau := \frac{12L\eta^2 /\alpha + \ell/(m\mu)}{\eta -12L\eta^2/\alpha} <1,
\end{equation}
then the VRSMD has a stronger linear convergence rate:
\begin{equation}\label{ch3:eq:thm2_linear}
    \mathbbm{E}[F(\boldsymbol{\beta}_a) - F(\boldsymbol{\beta}^*)] \leq \tau^S [F(\Tilde{\boldsymbol{\beta}}^0) - F(\boldsymbol{\beta}^*)].
\end{equation}
\end{proposition}
\begin{remark}
We analyze the computational complexity implied by Proposition \ref{CH3:THM:02}: 
In (a), let $m=n$ and take $\eta = \frac{\alpha}{48L}$, then we have 
\begin{align*}
    \mathbbm{E}&[F(\boldsymbol{\beta}_a) - F(\boldsymbol{\beta}^*)] \leq \frac{96L}{\alpha T}\times\\
    &\left[D_{\psi}(\boldsymbol{\beta}^*,\Tilde{\boldsymbol{\beta}}^{0})+ \frac{\alpha n}{192 L}(F(\Tilde{\boldsymbol{\beta}}^{0}) - F(\boldsymbol{\beta}^*))\right].
\end{align*}
In this case, the number of gradient computations for achieving an $\epsilon$-solution is $\mathcal{O}(\frac{L}{\epsilon}+ \frac{n}{\epsilon})$.
\end{remark}
%
%
\begin{remark}\label{ch3:remark:04}
The assumption in (b) is moderate. Take $m = \frac{110 L \ell}{\alpha\mu}$ and $\eta = \frac{\alpha}{36L}$, we have $\tau < 1$. 
This choice of $m$ does not violate the $\mathcal{O}(1)$ SFO comlexity per iteration -- 
take a good mirror map so that $\ell/\alpha = \mathcal{O}(1)$, and consider the most indicative case \cite{johnson2013accelerating} where the condition number $L/\mu = n$, we have $m = \Theta(n)$.
\end{remark}

Our results in Proposition \ref{CH3:THM:02} are consistent with those for SVRG. 
In part (a), the $\mathcal{O}(1/T)$ convergence rate matches the rate in \cite{reddi2016stochastic}; in part (b), the linear rate matches the rate in \cite{johnson2013accelerating}, while we reduce their strong convexity assumption to quadratic growth.

Finally, we are ready to present our main result on implicit regularization. 
In the following theorem, we show that VRSMD finds an $\epsilon$-solution of the minimum mirror interpolation problem. 
\begin{theorem} \label{CH3:THM:03}
For the objective function in \eqref{ch3:eq:sec2eq2}, assume $\psi(\boldsymbol{\beta})$ is $\alpha$-strongly convex w.r.t. $\|\cdot\|_2$, denote $L = \max_i\|\mathbf{x}_i\|_2^2$ and let $s_m$ be the smallest nonzero singular value of $X$. The VRSMD algorithm converges to the minimum mirror map interpolant 
\begin{align}
\begin{split}\label{ch3:eq:stronglyconvexobj2_special}
    \boldsymbol{\beta}^{\psi} :=\argmin_{\boldsymbol{\beta}}&\quad\psi(\boldsymbol{\beta})\\
    s.t. &\quad F(\boldsymbol{\beta}) = \min_{\boldsymbol{\beta}'} F(\boldsymbol{\beta}').
\end{split}
\end{align}
We describe the convergence by the following $\epsilon$-solution:

\textbf{(a)} Run Option I of Algorithm \ref{ch3:algo:SMD} with choice $\eta <\frac{\alpha}{24L}$ and initialization $\Tilde{\boldsymbol{\beta}}^{0} $ such that $\nabla \psi(\Tilde{\boldsymbol{\beta}}^{0})\in\mathrm{col}(X^T) $, assume that the output $\boldsymbol{\beta}^a$ satisfies $\|\nabla\psi(\boldsymbol{\beta}^a)\|_2 \leq B$, then we will have
\begin{align}
    \begin{split}\label{ch3:eq:thm3}
    &\mathbbm{E} [\psi (\boldsymbol{\beta}^a) - \psi(\boldsymbol{\beta}^{\psi})]\leq\frac{B}{s_m} \sqrt{\frac{\alpha}{(\alpha\eta-24L\eta^2)T}}\times\\
    &\left[2n D_{\psi}(\boldsymbol{\beta}^{\psi},\Tilde{\boldsymbol{\beta}}^{0})+ \frac{24 n L\eta^2 m}{\alpha}\left(F(\Tilde{\boldsymbol{\beta}}^{0}) - F(\boldsymbol{\beta}^{\psi})\right)\right]^{.5},
    \end{split}
\end{align}
which describes how far is the objective value at $\boldsymbol{\beta}^a$ away from the optimal solution to \eqref{ch3:eq:stronglyconvexobj2_special}. 
Moreover, we have
\begin{align}
    \begin{split}\label{ch3:eq:13}
    &\mathbbm{E}[F(\boldsymbol{\beta}^{a}) - F(\boldsymbol{\beta}^{\psi})] 
    \leq \frac{\alpha}{(\alpha\eta-8L\eta^2)T} \times\\&\left[ D_{\psi}(\boldsymbol{\beta}^{\psi},\Tilde{\boldsymbol{\beta}}^{0})+ \frac{12L\eta^2 m}{\alpha}\left(F(\Tilde{\boldsymbol{\beta}}^{0}) -  F(\boldsymbol{\beta}^{\psi})\right)\right],
    \end{split}
\end{align}
which characterizes how much does $\boldsymbol{\beta}^a$ violates the constraints of \eqref{ch3:eq:stronglyconvexobj2_special}. They together show that the VRSMD algorithm finds an $\epsilon$-solution to \eqref{ch3:eq:stronglyconvexobj2_special} for $T= \mathcal{O}(\frac{1}{\epsilon} + \frac{1}{\epsilon^2})$.

\textbf{(b)} Further assume that $\psi(\cdot)$ is $\ell$-smooth w.r.t. $\|\cdot\|_2$ and 
\begin{equation}
    \tau' = \frac{ 12L\eta^2 /\alpha + \ell n/(m s_m^2)}{\eta - 12L\eta^2/\alpha} < 1.
\end{equation}
Run Option II of Algorithm \ref{ch3:algo:SMD}, we can show that:
\begin{align}
    \mathbbm{E} [\psi (\boldsymbol{\beta}^a) - \psi(\boldsymbol{\beta}^{\psi})]&\leq \frac{B(\tau')^{S/2}\sqrt{2n}}{s_m} \sqrt{F(\Tilde{\boldsymbol{\beta}}^{0})-  F(\boldsymbol{\beta}^{\psi})},\nonumber\\
    \mathbbm{E}[F(\boldsymbol{\beta}^{a}) -F(\boldsymbol{\beta}^{\psi})] &\leq (\tau')^S\left(F(\Tilde{\boldsymbol{\beta}}^{0}) -F(\boldsymbol{\beta}^{\psi})\right).\label{ch3:eq:thm3:part2}
\end{align}

\end{theorem}
\begin{remark} 
We need to point out that the assumptions in Theorem \ref{CH3:THM:03} are moderate. 
For instance, for the assumption that $\nabla \psi(\Tilde{\boldsymbol{\beta}}^{0})\in\mathrm{col}(X^T)$, we can take $\Tilde{\boldsymbol{\beta}}^{0} = (\nabla \psi)^{-1}(X^T \mathbf{a})$ for any $\mathbf{a}$. 
One feasible choice is $\mathbf{a} = \mathbf{0}$, resulting in $\tilde{\boldsymbol{\beta}}_0 = \arg\min_{\boldsymbol{\beta}\in R^p}\psi(\boldsymbol{\beta})$. 
Since $\psi$ is strongly convex, this minimizer is not hard to calculate, for example: we have \\
$\cdot \psi(\cdot) = \|\cdot\|_{q}^q$ or $\psi(\cdot) = \|\cdot\|_{q}^2$ for $q> 1 \Rightarrow$ $\arg\min \psi(\cdot) = \mathbf{0}$;\\
$\cdot \psi(\boldsymbol{\beta}) = \boldsymbol{\beta}^T H \boldsymbol{\beta}$ for a positive definite $H$ $\Rightarrow \arg\min \psi(\cdot) = \mathbf{0}$;\\
$\cdot \psi(\boldsymbol{\beta}) = \sum_{i=1}^p {\beta}_i\log({\beta}_i) - {\beta}_i$ $\Rightarrow\arg\min \psi(\cdot) = \mathbf{1}$.
    
\end{remark}
\begin{remark}
As for the assumption in (b), we can take $m = \frac{110 L \ell n}{\alpha s_m^2}$ and $\eta = \frac{\alpha}{36L}$ to get 
$\tau' = (1 + 108/110)/2 < 1$. Take a good mirror map such that $\ell/\alpha = \mathcal{O}(1)$ and assume 
 $L/s_m^2 = \mathcal{O}(1)$, we further have $m =\Theta(n)$, so the algorithm can be implemented efficiently.
\end{remark}

It is useful to provide a high level understanding of Theorem \ref{CH3:THM:03}. 
It implies that the discrete update of the VRSMD Algorithm on the unregularized objective \eqref{ch3:eq:sec2eq2} is an $\epsilon$-solution of the regularized optimization problem \eqref{ch3:eq:stronglyconvexobj2_special}, so it is an implicit regularization result. 
Furthermore, since \eqref{ch3:eq:stronglyconvexobj2_special} minimizes a strictly convex function over a convex set, the solution will be unique, thus the estimator from VRSMD must converge to this unique solution. 
The finding in Theorem \ref{CH3:THM:03} can be further extended to other variants of SMD.

\section{Further Understanding of Implicit Regularization}\label{ch3:sec:04}
In this section, we provide a deeper understanding of our theoretical results in the previous section by analyzing two special mirror maps for linear regression model: one is $\psi(\boldsymbol{\beta}) = \|\boldsymbol{\beta}\|_{2}^{2}/2$, the other is $\psi(\boldsymbol{\beta}) = \|\boldsymbol{\beta}\|_{1+\delta}^{1+\delta}$ for small $\delta>0$.

Let us first consider $\psi(\cdot) = \|\cdot\|_2^2/2$, where VRSMD reduces to SVRG algorithm in \cite{johnson2013accelerating}. In this case, we further show that $\|\boldsymbol{\beta}^a - \boldsymbol{\beta}^{\psi}\|_2^2$ linearly converges to $0$, where $\boldsymbol{\beta}^{\psi}$ is the minimum $\ell_2$ norm solution $\boldsymbol{\beta}^{\psi} = (X^T X)^{+} X^T \mathbf{y} = X^+ \mathbf{y}$:

\begin{corollary}\label{CH3:COR:01}
Denote $L=\max_i\|\mathbf{x}_i\|_2^2$, take $\psi(\cdot) = \|\cdot\|_2^2/2$, let $\Tilde{\boldsymbol{\beta}}^{0} = \boldsymbol{\beta}_{m}^{0} = \mathbf{0}$, $\eta_t = \eta <1/(24L)$, and assume
\begin{equation}
\tau'' = \frac{ 12L\eta^2  + n/(m s_m^2)}{\eta - 12L\eta^2} <1.
\end{equation}
Run Option II of Algorithm \ref{ch3:algo:SMD} on \eqref{ch3:eq:sec2eq2}, we have
\begin{equation}
\mathbbm{E}\|\boldsymbol{\beta}^a -X^+ \mathbf{y}\|_2^2\leq  \frac{(\tau'')^S}{s_m^2}\|P_{\mathrm{col}(X)}\mathbf{y}\|_2^2.
\end{equation}
\end{corollary}

Next, let us consider the mirror map $\psi(\boldsymbol{\beta}) = \|\boldsymbol{\beta}\|_{1+\delta}^{1+\delta}$ for a small $\delta > 0$ in VRSMD, which leads to a sparse solution. 
Assume that $\|\boldsymbol{\beta}_t^s\|_{\infty} \leq K$ for a large enough $K$ throughout the updates, we then have $\psi(\boldsymbol{\beta}_t^s)$ is $\frac{(1+\delta)\delta}{K^{1-\delta}}$-strongly convex and $\|\nabla \psi(\boldsymbol{\beta}^{a})\|_2\leq \sqrt{p}(1+\delta)K^{\delta}$. 
By Theorem \ref{CH3:THM:03} we have the estimator sequence converges to the penalized solution 
\begin{align}
\boldsymbol{\beta}^{(\delta)}=\argmin_{\boldsymbol{\beta}}\left\{ \|\boldsymbol{\beta}\|_{1+\delta}^{1+\delta}:  X\boldsymbol{\beta} = P_{\mathrm{col}(X)} \mathbf{y}\right\}.
\end{align}
This choice of mirror map has $\lim_{\delta\to0}\|\boldsymbol{\beta}\|_{1+\delta}^{1+\delta} = \|\boldsymbol{\beta}\|_{1}$ and the solution $\boldsymbol{\beta}^{(0)} := \argmin_{\boldsymbol{\beta}}\{ \|\boldsymbol{\beta}\|_{1}:  X\boldsymbol{\beta} = P_{\mathrm{col}(X)} \mathbf{y}\}$ is sparse. Thus we expect that for small $\delta$, $\boldsymbol{\beta}^{(\delta)}$ is close to $\boldsymbol{\beta}^{(0)} $ and recovers a sparse true parameter. 

We now provide a rigorous argument for the sparse recovery. Assume that the data is generated by $\mathbf{y} = X \boldsymbol{\beta}^o $ for a sparse $\boldsymbol{\beta}^o$. In the following theorem we have $\boldsymbol{\beta}^{(\delta)}$ accurately recovers $\boldsymbol{\beta}^o$ when the design matrix $X$ satisfies some proper conditions. 

\begin{theorem}[Sparse Recovery]\label{CH3:COR:02}
Under the sparse setting defined above, denote $s = \|\boldsymbol{\beta}^o\|_0$. Assume that the design matrix $X$ satisfies $(s,\gamma)$-RE condition and is $s$-good with constant $\kappa < \frac{1}{2}$. For any $\xi>0$, if we choose
\begin{equation}
\delta\leq \frac{\log\left(1+ \frac{(1-2\kappa)\sqrt{n\gamma}}{\sqrt{s}\|\mathbf{y}\|_2}\xi\right)}{\log p - \log\left(1+ \frac{(1-2\kappa)\sqrt{n\gamma}}{\sqrt{s}\|\mathbf{y}\|_2}\xi\right)},
\end{equation} 
we have
\begin{equation}
       \|\boldsymbol{\beta}^{(\delta)} - \boldsymbol{\beta}^o\|_1 \leq  \xi .
\end{equation}
\end{theorem}
By Theorem \ref{CH3:COR:02}, the estimator $\boldsymbol{\beta}^{(\delta)}$ estimates the sparse truth with a small error $\xi$. In this way, VRSMD algorithm \ref{ch3:algo:SMD} achieves near sparse recovery via implicit regularization. 
\begin{figure}[t]
    \centering
    \includegraphics[height = 3in,width = 3.2in]{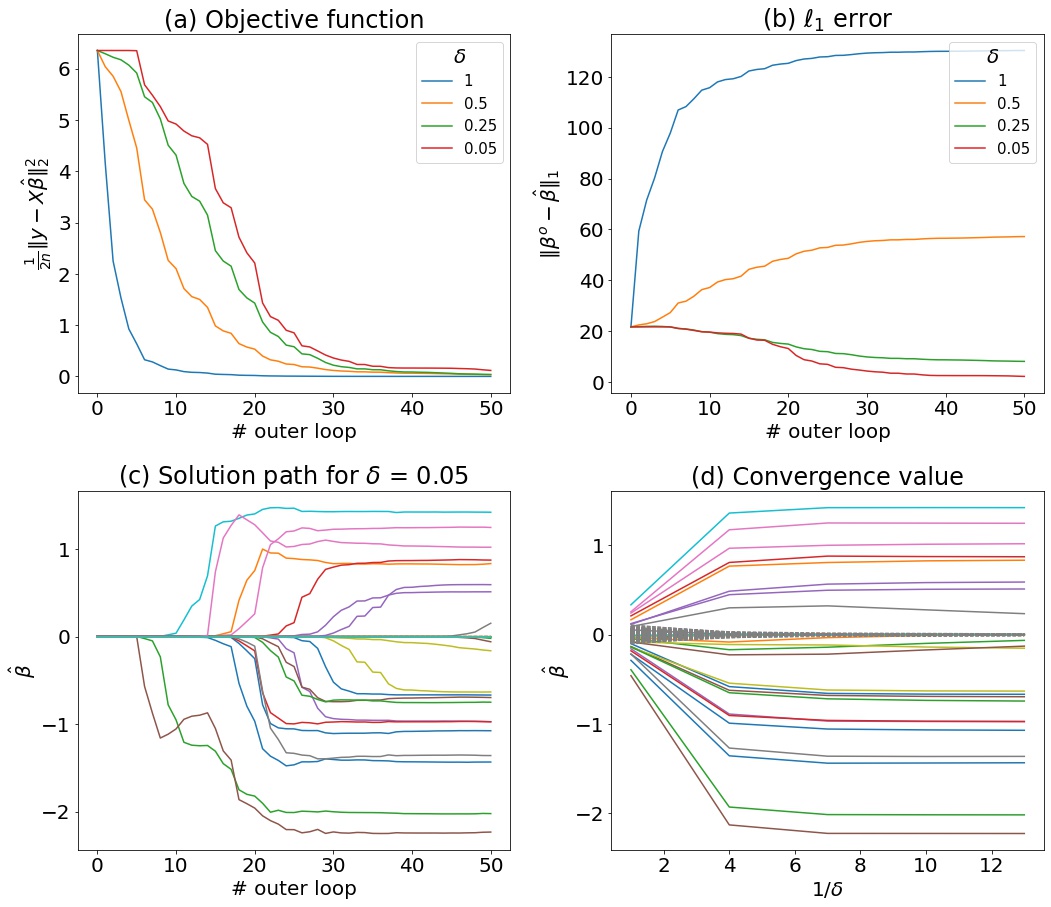}
    \caption{Run VRSMD Algorithm on simulated noiseless data. (a) the squared error objective function converges quickly to $0$. In (b), the $\ell_1$ estimation error converges to smaller value for smaller $\delta$, indicating the nearly exact recovery of the sparse signal, which supports Theorem \ref{CH3:COR:02}. In (c), for $\psi(\cdot) = \|\cdot\|_{1.05}^{1.05}$, the VRSMD estimator converges to the true parameter values. In (d), for a smaller $\delta$, the convergence value of VRSMD estimator is closer to ground truth.}
\label{ch3:fig:sparse_noiseless}
\end{figure}

\section{Numerical Experiment}\label{ch3:sec:05}
\textbf{Simulation.} We generate data by a sparse model 
as follows: 
Set $n = 1000,p = 5000 > n$ for the design matrix X. Simulate 
$X = \Sigma^{1/2} W$ where the entries of $W$ are i.i.d. $N(0,1)$ and $\Sigma = 0.5* \operatorname{diag}(\mathbf{1}_n) + 0.5* \mathbf{1}_{n\times n}$. The true parameter $\boldsymbol{\beta}^o \in \mathbbm{R}^p$ has its first $30$ entries sampled from i.i.d. $N(0,1)$ and the rest entries set to $0$. Compute responses $\mathbf{y} = X\boldsymbol{\beta}^o$. 

We then run VRSMD on objective function \eqref{ch3:eq:sec2eq2} for this simulated $X$ and $\mathbf{y}$. For a range of $\delta$, set the mirror map as $\psi(\cdot) = \|\cdot\|_{1+\delta}^{1+\delta}$, and run VRSMD with initialization $\Tilde{\boldsymbol{\beta}}^0 =\mathbf{0}$, step-size $\eta = 0.0002$, outer iteration number $50$ and inner iteration number $1000 = n$. The result is in Fig. \ref{ch3:fig:sparse_noiseless}.

\textbf{Experiment on RNA dataset.} 
We use the gene expression cancer RNA-Seq
data set\footnote{\url{https://archive.ics.uci.edu/ml/datasets/gene+expression+cancer+RNA-Seq}} for experiment. The data consists of $801$ observations, each of dimension $20,531$. Randomly split the data into $600$ training data and $201$ testing data. Run VRSMD algorithm on training data using mirror function $\psi = \|\cdot\|_{1.1}^{1.1}$ where initialization $\Tilde{\boldsymbol{\beta}}^0 =\mathbf{0}$, step-size $\eta = 0.015$, inner iteration number $400$ and outer iteration number determined by $5$-fold cross validation (i.e. early stopping). We also compare VRSMD with the Hadamard GD \cite{zhao2019implicit,NEURIPS2019_5cf21ce3}, which also has implicit regularization for sparsity. The result is in Fig. \ref{ch3:fig:real_data_compare_cv}.

\begin{figure}[ht]
    \centering
    \includegraphics[width = 3.2in,height = 2.5in]{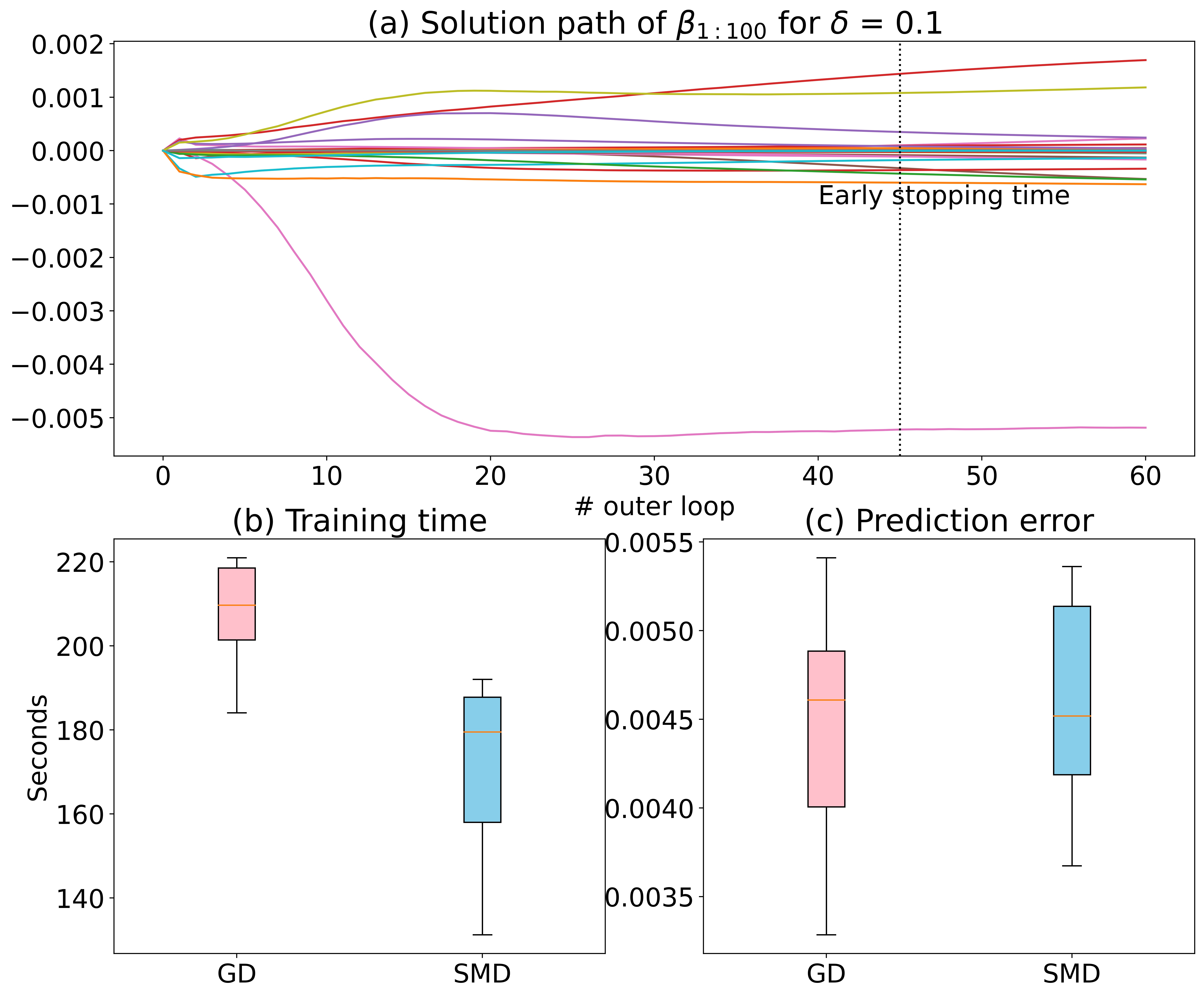}
    \caption{Run VRSMD Algorithm on RNA dataset. In (a), we plot the solution path of the first $100$ entries of $\boldsymbol{\beta}$, and it shows that the early stopped estimator is sparse. In (b) and (c), we compare the performance of VRSMD with Hadamard GD. Plot (b) shows that VRSMD trains faster than Hadamard GD, which is tested significant by one-sided Wilcoxon signed-rank test (p-value = $9.77\times 10^{-4}$). Plot (c) shows that the two algorithms have same prediction error on testing data, for which we test by two-sided Wilcoxon signed-rank test (p-value = $0.56$). 
    }
    \label{ch3:fig:real_data_compare_cv}
\end{figure}

\section{Conclusions and Future Research}\label{ch3:sec:06}
Our work analyzes the implicit regularization property of VRSMD, which covers both underfitting and overfitting cases in linear regression. In particular, our theorem shows that the implicit regularization property can help VRSMD find a sparse ground truth with a small error. Our experiments illustrate that the VRSMD is computationally efficient compared to the Hadamard GD algorithm that has implicit regularization for sparsity \cite{zhao2019implicit}, thanks to the stochastic nature of VRSMD. 

We discuss some future directions of our research. 
First, it is useful to study the implicit regularization properties of the VRSMD in the nonEuclidean setup. For example, one can consider the generalized linear model (GLM) where the data lies in a Riemannian manifold and mirror descent and/or natural gradient descent are efficient. Our analysis 
can be extended from the linear regression model to the GLM case. 

Second, it is interesting to investigate the minimax property of the VRSMD estimator. 
We see from our experiment that VRSMD with early stopping has comparable prediction performance to Hadamard GD that is minimax optimal for sparse regression \cite{zhao2019implicit, NEURIPS2019_5cf21ce3}. This leads to the open question of how to select a good mirror map and an optimal stopping time in VRSMD (or variants of VRSMD) such that the resulting estimator is minimax optimal and thus generalizes well.

Third, one can explore the implicit regularization properties of other variants of SMD. For example, one can consider the accelerated version of VRSMD  as the Katyusha algorithm in \cite{allenzhu2018katyusha}. Such an algorithm is well studies in optimization literature, and it has a better convergence rate that can match the theoretical optimum. We hypothesize that it also enjoys the implicit regularization property, but the proof is out of the scope of this work.   

Finally, our results might provide a better understanding of deep neural networks.  By \cite{NEURIPS2020_024d2d69}, Gradient Descent on Hadamard reparameterized linear regression, which is related to a neural network with multiple layers, can be approximated by Mirror Descent on original parameters. This point of view allows us to study a neural network from the VRSMD perspective and helps to explain why the Gradient Descent gives a sparse estimator in some deep learning models.

\printbibliography

\appendix
\section{Comparison with the VRSMD work \cite{li2021variance}}\label{ch3:app:background}
Besides the implicit regularization property that is emphasised in this work, the variance reduced stochastic mirror descent (VRSMD) itself is an algorithm of independent interest to the optimization community. We notice a concurrent work that also applies variance reduction to SMD \cite{li2021variance}. Our work is independently done without been affected by their work, and we have different focuses. The high level difference of their work and this one is that: their main goal is the convergence analysis, while this work focus on the interesting implicit regularization property and its implication of VRSMD. Nevertheless, we can compare the overlapping part of our work with theirs, which is the convergence analysis. Here we do it from three perspectives: assumptions, measurement of convergence, and the the number of stochastic gradient $\nabla f_i(x)$ (i.e. SFO calls) required to achieve a $\epsilon$ convergence:

\textbf{Difference in assumption:} There are 4 main difference in assumptions:

1. They assume their objective function to be $\min \frac{1}{n}\sum_i f_i(x) + g(x)$, where $g$ is possibly non-smooth but convex and $f_i$ is possibly non-convex but smooth. Our formulation does not have the non-smooth component, and we assume $f_i$ to be convex. In their work, they assume the proximal operator in $g$ can be efficiently solved to handle the non-smoothness. And they use different measurement for convergence compared to our work due to non-convexity. 

2. Their assumptions on the strong convexity of mirror map and the smoothness of the objective $f_i$s are with respect to $\ell_2$ norm. Our assumptions of the strong convexity and smoothness are for an arbitrary norm. Considering that the mirror descent is a nonEuclidean variant of gradient descent, assume a nonEuclidean norm would be more suitable.

3. They assume that the $f_i$ has bounded variance as $E_i[\|\nabla f_i(x) - \nabla f(x)\|_2^2]\leq \sigma^2$. This assumption is not required in our work.

4. They assume the PL inequality (w.r.t. the $\ell_2$ norm) holds to achieve a linear convergence rate. We prove the same rate, but under Quadratic Growth (QC) condition (w.r.t. an arbitrary norm). When considering the $\ell_2$ norm, the PL condition implies the QC condition \cite{karimi2016linear}, thus their assumption is stronger than our assumption here.

\textbf{Difference in measuring convergence:} Since the objective in \cite{li2021variance} is nonconvex, the convergence is evaluated for the stability gap $E[\|\nabla f(x)\|_2^2]$. Our works applies to a convex objective, so the convergence is evaluated for the optimization gap $E[f(x)] - \min_x f(x)$. 

\textbf{Comparing the SFO complexity:}
Unify all other notations and conditions in our work and theirs, we compare the SFO complexity required to achieve a $\epsilon$ stability gap, i.e. $E[\|\nabla f(x)\|_2^2]\leq \epsilon$:

For the general case, our SFO is $\mathcal{O}(\frac{1}{\epsilon}+ \frac{n}{\epsilon})$, and theirs is $\mathcal{O}(\min(\frac{n}{\epsilon},\frac{\sigma^2}{\epsilon^2}) + \frac{1}{\epsilon})$. Ignoring the bounded variance assumption by taking $\sigma = \infty$, then their rate is the same as ours; 

For the special case where we have linear convergence rate, our SFO is $\mathcal{O}(n\log\frac{1}{\epsilon}+\frac{1}{\mu}\log\frac{1}{\epsilon})$, and theirs is $\mathcal{O}(\min(n,\frac{\sigma^2}{\mu\epsilon})\frac{1}{\mu}\log\frac{1}{\epsilon} + \frac{1}{\mu}\log\frac{1}{\epsilon})$. Again taking $\sigma = \infty$, our SFO complexity is better than theirs, possibly because the difference in problem setting.

\section{Lemmas}\label{ch3:app:lem}
We first list some lemmas that will be used in the proof. For each lemma, we either give its reference or show the proof.

\begin{lemma}[\cite{Sidford2019}]\label{ch3:lem:lsmooth_quad_upper}
For a $L$-smooth function $f$ with respect to $\|\cdot\|$, we have:
\begin{equation}\label{ch3:eq:lsmooth_quad_upper}
     f(\mathbf{u}) - f(\mathbf{w}) - \nabla f(\mathbf{w})^T (\mathbf{u} - \mathbf{w})\leq \frac{L}{2} \|\mathbf{u} - \mathbf{w}\|^2.
\end{equation}
\end{lemma}

\begin{lemma}[Theorem 2.1.5 in \cite{nesterov1998introductory}]\label{ch3:lem:0}
For a $L$-smooth function $f$ with respect to $\|\cdot\|$, we have:
\begin{equation}\label{ch3:eq:lem0_0}
    \|\nabla f(\mathbf{u}) - \nabla f(\mathbf{w})\|^2_* \leq 2L [f(\mathbf{u}) - f(\mathbf{w}) - \langle\nabla f(\mathbf{w}), \mathbf{u}- \mathbf{w}\rangle].
\end{equation}
And especially, when $\nabla f(\mathbf{w}) = 0$:
\begin{equation}\label{ch3:eq:lem0}
    \|\nabla f(\mathbf{u})\|^2_* \leq 2L [f(\mathbf{u}) - f(\mathbf{w})].
\end{equation}
\end{lemma}

\begin{lemma}\label{ch3:lem:1}
For the problem $\min_{\boldsymbol{\beta}} F(\boldsymbol{\beta}) = \frac{1}{n}\sum_{i=1}^n f_i(\boldsymbol{\beta})$, suppose all $f_i$s are convex and $L-$smooth w.r.t. norm $\|\cdot\|$, denote $\boldsymbol{\beta}^* = \argmin_{\boldsymbol{\beta}} F(\boldsymbol{\beta})$, then for $i$ chosen uniformly random from $\{1,\ldots,n\}$, the term $\mathbf{v} =\nabla f_{i}(\boldsymbol{\beta}) - \nabla f_{i}(\Tilde{\boldsymbol{\beta}}) + \nabla F(\Tilde{\boldsymbol{\beta}})$ have:
\begin{equation}\label{ch3:eq:lem1}
    \mathbbm{E} \|\mathbf{v}\|_*^2 \leq 12 L [F(\boldsymbol{\beta}) - F(\boldsymbol{\beta}^*) + F(\Tilde{\boldsymbol{\beta}}) - F(\boldsymbol{\beta}^*)].
\end{equation}
\end{lemma}
\begin{proof}
\begin{align*}
    \mathbbm{E} \|\mathbf{v}\|_*^2 & = \mathbbm{E} \|\nabla f_{i}(\boldsymbol{\beta}) - \nabla f_{i}(\Tilde{\boldsymbol{\beta}}) + \nabla F(\Tilde{\boldsymbol{\beta}})\|_*^2\\
    &\leq 3\mathbbm{E} \|\nabla f_{i}(\boldsymbol{\beta}) -\nabla f_{i}(\boldsymbol{\beta}^*)\|_*^2 + 3\mathbbm{E} \|\nabla f_{i}(\boldsymbol{\beta}^*) - \nabla f_{i}(\Tilde{\boldsymbol{\beta}})\|_*^2 + 3\|\nabla F(\Tilde{\boldsymbol{\beta}})\|_*^2\\
    &= 3\mathbbm{E} \|\nabla f_{i}(\boldsymbol{\beta}) -\nabla f_{i}(\boldsymbol{\beta}^*)\|_*^2 + 3\mathbbm{E} \|\nabla f_{i}(\boldsymbol{\beta}^*) - \nabla f_{i}(\Tilde{\boldsymbol{\beta}})\|_*^2 + 3\|\nabla F(\Tilde{\boldsymbol{\beta}}) - \nabla F(\boldsymbol{\beta}^*)\|_*^2\\
    &\stackrel{\eqref{ch3:eq:lem0_0}}{\leq} 6\mathbbm{E}\{L [f_i(\boldsymbol{\beta}) - f_i(\boldsymbol{\beta}^*) - \langle\nabla f_i(\boldsymbol{\beta}^*), \boldsymbol{\beta}- \boldsymbol{\beta}^*\rangle]\} + 6\mathbbm{E}\{L [f_i(\Tilde{\boldsymbol{\beta}}) - f_i(\boldsymbol{\beta}^*) - \langle\nabla f_i(\boldsymbol{\beta}^*), \Tilde{\boldsymbol{\beta}}- \boldsymbol{\beta}^*\rangle]\}\\
    &\quad + 6\{L [F(\Tilde{\boldsymbol{\beta}}) - F(\boldsymbol{\beta}^*) - \langle\nabla F(\boldsymbol{\beta}^*), \Tilde{\boldsymbol{\beta}}- \boldsymbol{\beta}^*\rangle]\}\\
    &=  6L [F({\boldsymbol{\beta}}) - F(\boldsymbol{\beta}^*)] + 12L [F(\Tilde{\boldsymbol{\beta}}) - F(\boldsymbol{\beta}^*)]\\
    &\leq 12L [F({\boldsymbol{\beta}}) - F(\boldsymbol{\beta}^*) + F(\Tilde{\boldsymbol{\beta}}) - F(\boldsymbol{\beta}^*)]
\end{align*}

\end{proof}

\begin{lemma}\label{ch3:lem:3}
For the linear regression objective function $F(\boldsymbol{\beta})= \|X\boldsymbol{\beta} - y\|_2^2/(2n) =\frac{1}{n}\sum_{i=1}^n f_i(\boldsymbol{\beta})$ where $f_i(\boldsymbol{\beta}) = \left(\mathbf{x}_i^T\boldsymbol{\beta} - y_i\right)^2/2$, the term $\mathbf{v} =\nabla f_{i}(\boldsymbol{\beta}) - \nabla f_{i}(\Tilde{\boldsymbol{\beta}}) + \nabla F(\Tilde{\boldsymbol{\beta}})$ is in $\mathrm{col}(X^T)$.
\end{lemma}\begin{proof}
Check 
\begin{align*}
    \nabla f_{i}(\boldsymbol{\beta}) = \mathbf{x}_{i}(\mathbf{x}_{i}^T \boldsymbol{\beta} - y_{i}) = X^T \begin{bmatrix}\mathbf{0}_{i-1}\\1\\\mathbf{0}_{n-i}\end{bmatrix}(\mathbf{x}_{i}^T \boldsymbol{\beta} - y_{i}) \in\mathrm{col}(X^T).
\end{align*}
And $\nabla F(\Tilde{\boldsymbol{\beta}}) = X^T\left[\frac{1}{n}(X\Tilde{\boldsymbol{\beta}} - \mathbf{y})\right] \in\mathrm{col}(X^T)$. Thus $\mathbf{v} =\nabla f_{i}(\boldsymbol{\beta}) - \nabla f_{i}(\Tilde{\boldsymbol{\beta}}) + \nabla F(\Tilde{\boldsymbol{\beta}}) \in \mathrm{col}(X^T)$.
\end{proof}

\begin{lemma}\label{ch3:lem:4}
For a matrix $X\in \mathbbm{R}^{n\times p}$, for any $\mathbf{u} \in \mathrm{col}(X^T)$, we have
\begin{equation}\label{ch3:eq:lem4}
    \|X\mathbf{u}\|_2^2 \geq s_m^2\|\mathbf{u}\|_2^2
\end{equation}
where $s_m^2$ is the smallest non-zero eigenvalue of $X^T X$.
\end{lemma}
\begin{proof}
Consider the SVD of $X^T =  Q\Sigma V$ where $Q^T Q = Q Q^T = I_p$ and $V^T V = V V^T = I_n$
\begin{equation*}
    \Sigma = \begin{bmatrix}\operatorname{diag}(\mathbf{s})_m & \mathbf{0}_{m\times (n-m)}\\
    \mathbf{0}_{(p-m)\times m}&\mathbf{0}_{(p-m)\times(n-m)}
\end{bmatrix}_{p\times n}
\end{equation*}
with $\mathbf{s} = [s_1,\ldots,s_m]$, $s_1>\ldots>s_m>0$. And write $Q = [\mathbf{q}_1,\ldots,\mathbf{q}_p]$, then for any $\mathbf{u} \in \mathrm{col}(X^T)$, we can write $\mathbf{u} = \sum_{i=1}^m w_i\mathbf{q}_i$ for some $w_1,\ldots,w_m$, and $\|\mathbf{u}\|_2^2 = \sum_{i=1}^m w_i^2$.

Then
\begin{align*}
\begin{split}
    X\mathbf{u} &= V^T \Sigma^T Q^T (\sum_{i=1}^m w_i\mathbf{q}_i)\\
    &=V^T \Sigma^T \begin{bmatrix}w_1\\\vdots\\w_m\\0\\\vdots\\0\end{bmatrix}\\
    &=V^T \begin{bmatrix}s_1 w_1\\\vdots\\s_m w_m\\0\\\vdots\\0\end{bmatrix}.
\end{split}
\end{align*}
Thus 
\begin{align*}
    \begin{split}
        \|X\mathbf{u}\|_2^2 = \|V^T \begin{bmatrix}s_1 w_1\\\vdots\\s_m w_m\\0\\\vdots\\0\end{bmatrix}\|_2^2 = \sum_{i=1}^m s_i^2 w_i^2 \geq s_m^2 \sum_{i=1}^m w_i^2 = s_m^2 \|\mathbf{u}\|_2^2.
    \end{split}
\end{align*}
\end{proof}

\begin{lemma}\label{ch3:lem:5}
If $f$ is $\alpha-$strongly convex w.r.t. $\|\cdot\|$, then 
\begin{equation}\label{ch3:eq:lem5}
    f(\mathbf{u})\leq f(\mathbf{w}) + \nabla f(\mathbf{w}) ^T (\mathbf{u} - \mathbf{w}) + \frac{1}{2\alpha}\|\nabla f(\mathbf{u}) - \nabla f(\mathbf{w})\|_*^2.
\end{equation}
\end{lemma}
\begin{proof}
Consider $g(\mathbf{v}) =f(\mathbf{v}) - \mathbf{v}^T \nabla f(\mathbf{w})$, then $g$ is also $\alpha-$strongly convex, and it's easy to check that $g$ has a unique minimizer $\mathbf{w}$. 

Next, by $\alpha-$strong convexity of $g$, for any $\mathbf{u}$ we have
\begin{equation}\label{ch3:eq:lem5:eq1}
    g(\mathbf{v}) \geq g(\mathbf{u}) + \nabla g(\mathbf{u})^T (\mathbf{v} - \mathbf{u}) + \frac{\alpha}{2} \|\mathbf{v}-\mathbf{u}\|^2.
\end{equation}
Minimizing both side of \eqref{ch3:eq:lem5:eq1} w.r.t. $\mathbf{v}$, we have
\begin{align*}
\begin{split}
\label{ch3:eq:lem5:eq2}
    g(\mathbf{w}) = \min_{\mathbf{v}}g(\mathbf{v}) &\geq \min_{\mathbf{v}} g(\mathbf{u}) + \nabla g(\mathbf{u})^T (\mathbf{v} - \mathbf{u}) + \frac{\alpha}{2} \|\mathbf{v}-\mathbf{u}\|^2\\
    &\geq \min_{\mathbf{v}} g(\mathbf{u}) - \|\nabla g(\mathbf{u})\|_* \|\mathbf{v} - \mathbf{u}\| + \frac{\alpha}{2} \|\mathbf{v}-\mathbf{u}\|^2\\
    &\geq g(\mathbf{u}) - \frac{1}{2\alpha} \|\nabla g(\mathbf{u})\|_*^2.
\end{split}
\end{align*}
That is, 
\begin{equation*}
    f(\mathbf{w}) - \mathbf{w}^T \nabla f(\mathbf{w}) \geq f(\mathbf{u}) - \mathbf{u}^T \nabla f(\mathbf{w}) - \frac{1}{2\alpha} \|\nabla f(\mathbf{u}) - \nabla f(\mathbf{w})\|_*^2.
\end{equation*}
Rearrange the terms we get \eqref{ch3:eq:lem5}.
\end{proof}

\begin{lemma}\label{ch3:lem:6}
If $f$ is $\alpha-$strongly convex w.r.t. $\|\cdot\|$, then 
\begin{equation}\label{ch3:eq:lem6}
   [\nabla f(\mathbf{u}) - \nabla f(\mathbf{w})]^T[\mathbf{u} - \mathbf{w}] \leq\frac{1}{\alpha}\|\nabla f(\mathbf{u}) - \nabla f(\mathbf{w})\|_*^2.
\end{equation}
\end{lemma}
\begin{proof}
Apply \eqref{ch3:eq:lem5} twice to get:
\begin{align}
    f(\mathbf{u})&\leq f(\mathbf{w}) + \nabla f(\mathbf{w}) ^T (\mathbf{u} - \mathbf{w}) + \frac{1}{2\alpha}\|\nabla f(\mathbf{u}) - \nabla f(\mathbf{w})\|_*^2\label{ch3:eq:lem6:eq1},\\
    f(\mathbf{w})&\leq f(\mathbf{u}) + \nabla f(\mathbf{u}) ^T (\mathbf{w} - \mathbf{u}) + \frac{1}{2\alpha}\|\nabla f(\mathbf{u}) - \nabla f(\mathbf{w})\|_*^2\label{ch3:eq:lem6:eq2}.
\end{align}
Sum \eqref{ch3:eq:lem6:eq1} and \eqref{ch3:eq:lem6:eq2} we have:
\begin{align*}
    &0\leq (\nabla f(\mathbf{w}) - f(\mathbf{u})) ^T (\mathbf{u} - \mathbf{w})+ \frac{1}{\alpha}\|\nabla f(\mathbf{u}) - \nabla f(\mathbf{w})\|_*^2\\
    \Longrightarrow &   [\nabla f(\mathbf{u}) - \nabla f(\mathbf{w})]^T[\mathbf{u} - \mathbf{w}] \leq\frac{1}{\alpha}\|\nabla f(\mathbf{u}) - \nabla f(\mathbf{w})\|_*^2.
\end{align*}
\end{proof}

\begin{lemma}[\cite{karimi2016linear}]\label{ch3:lem:PL_QG}
If $f$ satisfies $PL$ inequality with respect to $\ell_2$ norm $\|\cdot\|_2$, then $f$ satisfies the Quadratic Growth (QG) condition w.r.t. $\|\cdot\|_2$ for the same constant $\mu$.
\end{lemma}

\section{Proof for Proposition \ref{CH3:THM:02}}\label{ch3:app:C}
\begin{proof}
Consider the Bregman divergence
\begin{align*}
        &\mathbbm{E}[D_{\psi}(\boldsymbol{\beta}^*,\boldsymbol{\beta}_{t+1}^{s})]\\
        =&\mathbbm{E}[\psi(\boldsymbol{\beta}^*) - \psi(\boldsymbol{\beta}_{t+1}^{s}) - \nabla \psi(\boldsymbol{\beta}_{t+1}^{s})^T (\boldsymbol{\beta}^*-\boldsymbol{\beta}_{t+1}^{s})]\\ 
        =&\mathbbm{E}[\psi(\boldsymbol{\beta}^*) - \psi(\boldsymbol{\beta}_{t+1}^{s}) - [\nabla \psi(\boldsymbol{\beta}_{t}^{s}) - \eta\mathbf{v}_t^{s}]^T (\boldsymbol{\beta}^* - \boldsymbol{\beta}_{t}^{s}+\boldsymbol{\beta}_{t}^{s} -\boldsymbol{\beta}_{t+1}^{s})]\\
        =&\mathbbm{E}[\psi(\boldsymbol{\beta}^*) - \psi(\boldsymbol{\beta}_{t}^{s}) + \psi(\boldsymbol{\beta}_{t}^{s}) - \psi(\boldsymbol{\beta}_{t+1}^{s}) - \nabla \psi(\boldsymbol{\beta}_{t}^{s})^T (\boldsymbol{\beta}^* - \boldsymbol{\beta}_{t}^{s})\\
        &+ \nabla \psi(\boldsymbol{\beta}_{t}^{s})^T (\boldsymbol{\beta}_{t+1}^{s} - \boldsymbol{\beta}_{t}^{s}) + (\eta\mathbf{v}_t^{s})^T (\boldsymbol{\beta}^* - \boldsymbol{\beta}_{t+1}^{s})]\\
        =& \mathbbm{E}[D_{\psi}(\boldsymbol{\beta}^*,\boldsymbol{\beta}_{t}^{s}) - D_{\psi}(\boldsymbol{\beta}_{t+1}^{s},\boldsymbol{\beta}_{t}^{s})+ (\eta\mathbf{v}_t^{s})^T (\boldsymbol{\beta}^* - \boldsymbol{\beta}_{t+1}^{s})].\numberthis\label{ch3:eq:thm2:eq1}
\end{align*}

The last term in \eqref{ch3:eq:thm2:eq1}:
\begin{align*}
    &\mathbbm{E}[(\mathbf{v}_t^{s})^T (\boldsymbol{\beta}^* - \boldsymbol{\beta}_{t+1}^{s})]\\
    =&\mathbbm{E}[(\mathbf{v}_t^{s})^T (\boldsymbol{\beta}^* - \boldsymbol{\beta}_{t}^{s} + \boldsymbol{\beta}_{t}^{s}-\boldsymbol{\beta}_{t+1}^{s})]\\
    =&\mathbbm{E}[\mathbbm{E}[(\mathbf{v}_t^{s})^T (\boldsymbol{\beta}^* - \boldsymbol{\beta}_{t}^{s})|\boldsymbol{\beta}_{t}^{s}]] + \mathbbm{E}[(\mathbf{v}_t^{s})^T(\boldsymbol{\beta}_{t}^{s}-\boldsymbol{\beta}_{t+1}^{s})]\\
    =&\mathbbm{E}[(\nabla F(\boldsymbol{\beta}_{t}^{s})^T (\boldsymbol{\beta}^* - \boldsymbol{\beta}_{t}^{s})] + \frac{1}{\eta}\mathbbm{E}[(\nabla \psi(\boldsymbol{\beta}_{t}^{s}) - \nabla  \psi(\boldsymbol{\beta}_{t+1}^{s}))^T(\boldsymbol{\beta}_{t}^{s}-\boldsymbol{\beta}_{t+1}^{s})]\\
    \leq&\mathbbm{E}[(F(\boldsymbol{\beta}^*) - F(\boldsymbol{\beta}_{t}^{s})] + \frac{1}{\eta}\mathbbm{E}[(\nabla \psi(\boldsymbol{\beta}_{t}^{s}) - \nabla  \psi(\boldsymbol{\beta}_{t+1}^{s}))^T(\boldsymbol{\beta}_{t}^{s}-\boldsymbol{\beta}_{t+1}^{s})]\\
    \stackrel{\eqref{ch3:eq:lem6}}{\leq} &\mathbbm{E}[(F(\boldsymbol{\beta}^*) - F(\boldsymbol{\beta}_{t}^{s})] + \frac{1}{\eta\alpha}\mathbbm{E}[\|\nabla \psi(\boldsymbol{\beta}_{t}^{s}) - \nabla  \psi(\boldsymbol{\beta}_{t+1}^{s})\|_*^2]\\
    =&\mathbbm{E}[(F(\boldsymbol{\beta}^*) - F(\boldsymbol{\beta}_{t}^{s})] + \frac{\eta}{\alpha}\mathbbm{E}[\|\mathbf{v}_t^{s}\|_*^2].
\end{align*}

Thus we have 
\begin{align*}
        &\mathbbm{E}[D_{\psi}(\boldsymbol{\beta}^*,\boldsymbol{\beta}_{t+1}^{s})]\\
        \leq& \mathbbm{E}[D_{\psi}(\boldsymbol{\beta}^*,\boldsymbol{\beta}_{t}^{s})] - \mathbbm{E}[D_{\psi}(\boldsymbol{\beta}_{t+1}^{s},\boldsymbol{\beta}_{t}^{s})]+\eta\mathbbm{E}[F(\boldsymbol{\beta}^*) - F(\boldsymbol{\beta}_{t}^{s})] + \frac{\eta^2}{\alpha}\mathbbm{E}[\|\mathbf{v}_t^{s}\|_*^2]\\
        \stackrel{\eqref{ch3:eq:lem1}}{\leq}& \mathbbm{E}[D_{\psi}(\boldsymbol{\beta}^*,\boldsymbol{\beta}_{t}^{s})] - \mathbbm{E}[D_{\psi}(\boldsymbol{\beta}_{t+1}^{s},\boldsymbol{\beta}_{t}^{s})]\\
        &\quad+(\eta - \frac{12L\eta^2}{\alpha})\mathbbm{E}[F(\boldsymbol{\beta}^*) - F(\boldsymbol{\beta}_{t}^{s})] + \frac{12L\eta^2}{\alpha}\mathbbm{E}[F(\Tilde{\boldsymbol{\beta}}^{s-1}) - F(\boldsymbol{\beta}^*)]\\
        \leq&\mathbbm{E}[D_{\psi}(\boldsymbol{\beta}^*,\boldsymbol{\beta}_{t}^{s})] + \left(\frac{24L\eta^2}{\alpha} - \eta\right)\mathbbm{E}[F(\boldsymbol{\beta}_{t}^{s}) -F(\boldsymbol{\beta}^*)]\\
        &\quad+ \frac{12L\eta^2}{\alpha}\mathbbm{E}[F(\Tilde{\boldsymbol{\beta}}^{s-1}) - F(\boldsymbol{\beta}^*)] - \frac{12L\eta^2}{\alpha}\mathbbm{E}[F(\boldsymbol{\beta}_{t}^{s}) -F(\boldsymbol{\beta}^*)].\numberthis\label{ch3:eq:thm2:eq3}
\end{align*}

Sum \eqref{ch3:eq:thm2:eq3} for $t = 1,\ldots,m$, we  have
\begin{align*}
\mathbbm{E}[D_{\psi}(\boldsymbol{\beta}^*,\boldsymbol{\beta}_{m+1}^{s})] \leq & \mathbbm{E}[D_{\psi}(\boldsymbol{\beta}^*,\boldsymbol{\beta}_{1}^{s})]+\sum_{t=1}^{m}\left(\frac{24L\eta^2}{\alpha} - \eta\right)\mathbbm{E}[F(\boldsymbol{\beta}_{t}^{s}) -F(\boldsymbol{\beta}^*)]\\
&\quad+ \frac{12L\eta^2 m}{\alpha}\mathbbm{E}[F(\Tilde{\boldsymbol{\beta}}^{s-1}) - F(\boldsymbol{\beta}^*)] - \frac{12L\eta^2}{\alpha}\sum_{t=1}^{m}\mathbbm{E}[F(\boldsymbol{\beta}_{t}^{s}) -F(\boldsymbol{\beta}^*)].\numberthis\label{ch3:eq:thm2:eq4}
\end{align*}
For option I,  $\boldsymbol{\beta}_{1}^{s}= \boldsymbol{\beta}_{m+1}^{s-1}$, $\Tilde{\boldsymbol{\beta}}^{s}$ is a uniform random sample from $\{\boldsymbol{\beta}_{1}^{s},\ldots,\boldsymbol{\beta}_{m}^{s}\}$,  \eqref{ch3:eq:thm2:eq4} becomes: 
\begin{align*}
            \mathbbm{E}[D_{\psi}(\boldsymbol{\beta}^*,\boldsymbol{\beta}_{m+1}^{s})]\leq&\mathbbm{E}[D_{\psi}(\boldsymbol{\beta}^*,\boldsymbol{\beta}_{m+1}^{s-1})]+\sum_{t=1}^{m}\left(\frac{24L\eta^2}{\alpha} - \eta\right)\mathbbm{E}[F(\boldsymbol{\beta}_{t}^{s}) -F(\boldsymbol{\beta}^*)]\\
        &\quad+ \frac{12L\eta^2 m}{\alpha}\mathbbm{E}[F(\Tilde{\boldsymbol{\beta}}^{s-1}) - F(\boldsymbol{\beta}^*)] - \frac{12L\eta^2 m}{\alpha}\mathbbm{E}[F(\Tilde{\boldsymbol{\beta}}^{s}) -F(\boldsymbol{\beta}^*)].
\end{align*}
That is,
\begin{align*}
        &\left(\eta-\frac{24L\eta^2}{\alpha}\right)\sum_{t=1}^{m}\mathbbm{E}[F(\boldsymbol{\beta}_{t}^{s}) -F(\boldsymbol{\beta}^*)]\\
        \leq& \mathbbm{E}[D_{\psi}(\boldsymbol{\beta}^*,\boldsymbol{\beta}_{m+1}^{s-1})]+ \frac{12L\eta^2 m}{\alpha}\mathbbm{E}[F(\Tilde{\boldsymbol{\beta}}^{s-1}) - F(\boldsymbol{\beta}^*)]\\
        &\quad-\mathbbm{E}[D_{\psi}(\boldsymbol{\beta}^*,\boldsymbol{\beta}_{m+1}^{s})] - \frac{12L\eta^2 m}{\alpha}\mathbbm{E}[F(\Tilde{\boldsymbol{\beta}}^{s}) -F(\boldsymbol{\beta}^*)].\numberthis \label{ch3:eq:thm2:eq5}
\end{align*}
Thus we can define Lyapunov function
\begin{align*}
    P^s:= \mathbbm{E}[D_{\psi}(\boldsymbol{\beta}^*,\boldsymbol{\beta}_{m+1}^{s})]+ \frac{12L\eta^2 m}{\alpha}\mathbbm{E}[F(\Tilde{\boldsymbol{\beta}}^s) - F(\boldsymbol{\beta}^*)]\geq 0.
\end{align*}
Then 
\begin{align*}
        \mathbbm{E}[F(\boldsymbol{\beta}_a) - F(\boldsymbol{\beta}^*)] &= \frac{1}{T}\sum_{s = 1}^{S}\sum_{t=1}^{m} \mathbbm{E}[F(\boldsymbol{\beta}_t^{s}) - F(\boldsymbol{\beta}^*)]\\
        &\stackrel{\eqref{ch3:eq:thm2:eq5}}{\leq} \frac{1}{T}\sum_{s = 1}^{S}\frac{1}{\eta-\frac{24L\eta^2}{\alpha}} [P^{s-1} - P^{s}]\\
        &\leq \frac{\alpha}{(\alpha\eta-24L\eta^2)T} P^0
\end{align*}
For Option II, $\boldsymbol{\beta}_{1}^{s}= \Tilde{\boldsymbol{\beta}}^{s-1}$,  \eqref{ch3:eq:thm2:eq4} becomes:
\begin{align*}
\mathbbm{E}[D_{\psi}(\boldsymbol{\beta}^*,\boldsymbol{\beta}_{m+1}^{s})]\leq&\mathbbm{E}[D_{\psi}(\boldsymbol{\beta}^*,\Tilde{\boldsymbol{\beta}}^{s-1})]+\sum_{t=1}^{m}\left(\frac{24L\eta^2}{\alpha} - \eta\right)\mathbbm{E}[F(\boldsymbol{\beta}_{t}^{s}) -F(\boldsymbol{\beta}^*)]\\
&\quad+ \frac{12L\eta^2 m}{\alpha}\mathbbm{E}[F(\Tilde{\boldsymbol{\beta}}^{s-1}) - F(\boldsymbol{\beta}^*)] - \frac{12L\eta^2 m}{\alpha}\mathbbm{E}[F(\Tilde{\boldsymbol{\beta}}^{s}) -F(\boldsymbol{\beta}^*)]\\
=&\mathbbm{E}[D_{\psi}(\boldsymbol{\beta}^*,\Tilde{\boldsymbol{\beta}}^{s-1})]+m\left(\frac{12L\eta^2}{\alpha} - \eta\right)\mathbbm{E}[F(\Tilde{\boldsymbol{\beta}}^{s}) -F(\boldsymbol{\beta}^*)]\\
&\quad+ \frac{12L\eta^2 m}{\alpha}\mathbbm{E}[F(\Tilde{\boldsymbol{\beta}}^{s-1}) - F(\boldsymbol{\beta}^*)].\numberthis \label{ch3:eq:thm2:eq6}
\end{align*}
Since $\psi$ is $\ell-$smooth and $F$ satisfies QG condition, $\boldsymbol{\beta}^*$ is any minimum point of $F$, we can take $\boldsymbol{\beta}^*$ as required by Lemma \ref{ch3:lem:PL_QG} to get
\begin{align*}
        D_{\psi}(\boldsymbol{\beta}^*,\Tilde{\boldsymbol{\beta}}^{s-1}) &= \psi(\boldsymbol{\beta}^*) - \psi(\Tilde{\boldsymbol{\beta}}^{s-1}) - \nabla \psi(\Tilde{\boldsymbol{\beta}}^{s-1})^T (\boldsymbol{\beta}^*-\Tilde{\boldsymbol{\beta}}^{s-1})\\
        &\stackrel{\eqref{ch3:eq:lsmooth_quad_upper}}{\leq}\frac{\ell}{2}\left\|\boldsymbol{\beta}^* - \Tilde{\boldsymbol{\beta}}^{s-1}\right\|^2 \\
        &\stackrel{QG\text{ of } f}{\leq } \frac{\ell}{\mu}[F(\Tilde{\boldsymbol{\beta}}^{s-1}) - F(\boldsymbol{\beta}^*)].\numberthis\label{ch3:eq:thm2:eq7}
\end{align*}
Plug \eqref{ch3:eq:thm2:eq7} into \eqref{ch3:eq:thm2:eq6} we have
\begin{align*}
        0 \leq \mathbbm{E}[D_{\psi}(\boldsymbol{\beta}^*,\boldsymbol{\beta}_{m+1}^{s})]\leq& \left(\frac{12L\eta^2m}{\alpha} - m\eta \right)\mathbbm{E}[F(\Tilde{\boldsymbol{\beta}}^{s}) -F(\boldsymbol{\beta}^*)]\\
        &\quad+ \left(\frac{12L\eta^2 m}{\alpha} + \frac{\ell}{\mu}\right)\mathbbm{E}[F(\Tilde{\boldsymbol{\beta}}^{s-1}) - F(\boldsymbol{\beta}^*)].
\end{align*}
That is, 
\begin{align*}
        \mathbbm{E}[F(\Tilde{\boldsymbol{\beta}}^{s}) -F(\boldsymbol{\beta}^*)]&\leq \frac{ \frac{12L\eta^2 }{\alpha} + \frac{\ell}{m\mu}}{\eta - \frac{12L\eta^2}{\alpha}}\mathbbm{E}[F(\Tilde{\boldsymbol{\beta}}^{s-1}) - F(\boldsymbol{\beta}^*)]\\
        &= \tau \mathbbm{E}[F(\Tilde{\boldsymbol{\beta}}^{s-1}) - F(\boldsymbol{\beta}^*)].
\end{align*}
Thus,
\begin{align*}
    \mathbbm{E}[F(\Tilde{\boldsymbol{\beta}}^{S}) -F(\boldsymbol{\beta}^*)]&\leq \tau^S   \mathbbm{E}[F(\Tilde{\boldsymbol{\beta}}^{0}) -F(\boldsymbol{\beta}^*)]
\end{align*}
and by $\boldsymbol{\beta}^a = \Tilde{\boldsymbol{\beta}}^{S}$ we get the desired bound.

\end{proof}

\section{Proof for Theorem \ref{CH3:THM:03}}\label{ch3:app:E}
\begin{proof}
We first check that every $f_i$ is $L-$smooth w.r.t. $\|\cdot\|_2$.
\begin{align*}
    \|\nabla f_i(\mathbf{u}) - \nabla f_i(\mathbf{w})\|_2 &= \|(\mathbf{x}_i^T\mathbf{u} - y_i)\mathbf{x}_i - (\mathbf{x}_i^T\mathbf{w} - y_i)\mathbf{x}_i\|_2\\
    &=\|\mathbf{x}_i^T(\mathbf{u} - \mathbf{w})\mathbf{x}_i\|_2\\
    &=|\mathbf{x}_i^T(\mathbf{u} - \mathbf{w})| * \|\mathbf{x}_i\|_2\\
    &\leq \|\mathbf{x}_i\|_2 \|\mathbf{u} - \mathbf{w}\|_2\|\mathbf{x}_i\|_2\\
    &\leq (\max_i \|x_i\|_2^2) \|\mathbf{u} - \mathbf{w}\|_2.
\end{align*}
First, we prove part (a) of Theorem \ref{CH3:THM:03}. When $\psi(\boldsymbol{\beta})$ is $\alpha-$strongly convex w.r.t. $\|\cdot\|_2$, we can apply part (a) of Theorem 1, take $\boldsymbol{\beta}^* = \boldsymbol{\beta}^{\psi}$ to get
\begin{align*}
        &\mathbbm{E}\left[\frac{1}{2n}\|X\boldsymbol{\beta}^{a} - \mathbf{y}\|_2^2 - \frac{1}{2n}\|X\boldsymbol{\beta}^{\psi} - \mathbf{y}\|_2^2\right] \\
        \leq& \frac{\alpha}{(\alpha\eta-24L\eta^2)T} \left[D_{\psi}(\boldsymbol{\beta}^{\psi},\Tilde{\boldsymbol{\beta}}^{0})+ \frac{12L\eta^2 m}{\alpha}\left(\frac{1}{2n}\|X\Tilde{\boldsymbol{\beta}}^{0} - \mathbf{y}\|_2^2 -\frac{1}{2n}\|X\boldsymbol{\beta}^{\psi} - \mathbf{y}\|_2^2\right)\right]\numberthis\label{ch3:eq:10}.
\end{align*}
For the L.H.S. of \eqref{ch3:eq:10} we have
\begin{align*}
        &\|X\boldsymbol{\beta}^{a} - \mathbf{y}\|_2^2 - \|X\boldsymbol{\beta}^{\psi} - \mathbf{y}\|_2^2\\ =& \|X\boldsymbol{\beta}^{a} - X\boldsymbol{\beta}^{\psi} + X\boldsymbol{\beta}^{\psi}- \mathbf{y}\|_2^2 - \|X\boldsymbol{\beta}^{\psi} - \mathbf{y}\|_2^2\\
        =&\|X\boldsymbol{\beta}^{a} - X\boldsymbol{\beta}^{\psi}\|_2^2 + \|X\boldsymbol{\beta}^{\psi}- \mathbf{y}\|_2^2 + 2\langle X\boldsymbol{\beta}^{a} - X\boldsymbol{\beta}^{\psi} ,P_{\mathrm{col}(X)}\mathbf{y}- \mathbf{y}\rangle- \|X\boldsymbol{\beta}^{\psi} - \mathbf{y}\|_2^2\\
        =&\|X\boldsymbol{\beta}^{a} - X\boldsymbol{\beta}^{\psi}\|_2^2  + 2\langle X(\boldsymbol{\beta}^{a} - \boldsymbol{\beta}^{\psi}),-P_{\mathcal{N}(X^T)}\mathbf{y}\rangle\\
        =& \|X\boldsymbol{\beta}^{a} - X\boldsymbol{\beta}^{\psi}\|_2^2=\|X\boldsymbol{\beta}^{a} - P_{\mathrm{col}(X)} \mathbf{y}\|_2^2.
\end{align*}

Thus \eqref{ch3:eq:10} becomes
\begin{align*}\label{ch3:eq:15}
       &\mathbbm{E}[\|X\boldsymbol{\beta}^{a} - X\boldsymbol{\beta}^{\psi}\|_2^2] = \mathbbm{E}[\|X\boldsymbol{\beta}^{a} - P_{\mathrm{col}(X)} \mathbf{y}\|_2^2] \\
       \leq& \frac{\alpha}{(\alpha\eta-24L\eta^2)T} \left[2n D_{\psi}(\boldsymbol{\beta}^{\psi},\Tilde{\boldsymbol{\beta}}^{0})+ \frac{12L\eta^2 m}{\alpha}\left(\|X\Tilde{\boldsymbol{\beta}}^{0} - \mathbf{y}\|_2^2 -\|X\boldsymbol{\beta}^{\psi} - \mathbf{y}\|_2^2\right)\right].\numberthis
\end{align*}
Since $\nabla \psi(\Tilde{\boldsymbol{\beta}}^{0}) =\nabla \psi( \boldsymbol{\beta}_{m}^{0})  \in \mathrm{col}(X^T)$, by Lemma \ref{ch3:lem:3} we will have $\nabla \psi(\boldsymbol{\beta}_t^{s}) \in \mathrm{col}(X^T)$ for all $s$ and $t$, thus
\begin{equation}\label{ch3:eq:16}
    \nabla \psi(\boldsymbol{\beta}^a) \in \mathrm{col}(X^T).
\end{equation}
Then
\begin{align*}
       &\mathbbm{E} [\psi (\boldsymbol{\beta}^a) - \psi(\boldsymbol{\beta}^{\psi})]\\\leq& \mathbbm{E}\langle \nabla \psi(\boldsymbol{\beta}^a), \boldsymbol{\beta}^a -\boldsymbol{\beta}^{\psi}\rangle\\
       \stackrel{\eqref{ch3:eq:16}}{=}&\mathbbm{E}\langle \nabla \psi(\boldsymbol{\beta}^a), P_{\mathrm{col}(X^T)}(\boldsymbol{\beta}^a -\boldsymbol{\beta}^{\psi})\rangle\\
       \leq& \mathbbm{E}[\|\nabla \psi(\boldsymbol{\beta}^a)\|_2 \|P_{\mathrm{col}(X^T)}(\boldsymbol{\beta}^a -\boldsymbol{\beta}^{\psi})\|_2]\\
       \stackrel{}{\leq}& B * \mathbbm{E}\|P_{\mathrm{col}(X^T)}(\boldsymbol{\beta}^a -\boldsymbol{\beta}^{\psi})\|_2\\
       \stackrel{\eqref{ch3:eq:lem4}}{\leq}&\frac{B}{s_m}  \mathbbm{E}\|X P_{\mathrm{col}(X^T)}(\boldsymbol{\beta}^a -\boldsymbol{\beta}^{\psi})\|_2\\
       =&\frac{B}{s_m}  \mathbbm{E}\|X\boldsymbol{\beta}^a -X\boldsymbol{\beta}^{\psi}\|_2\\
       \leq& \frac{B}{s_m}  (\mathbbm{E}\|X\boldsymbol{\beta}^a -X\boldsymbol{\beta}^{\psi}\|_2^2)^{1/2}\\
       \stackrel{\eqref{ch3:eq:15}}{\leq} &\frac{B}{s_m} \sqrt{\frac{\alpha}{(\alpha\eta-24L\eta^2)T} \left[2n D_{\psi}(\boldsymbol{\beta}^{\psi},\Tilde{\boldsymbol{\beta}}^{0})+ \frac{12L\eta^2 m}{\alpha}\left(\|X\Tilde{\boldsymbol{\beta}}^{0} - \mathbf{y}\|_2^2 -\|X\boldsymbol{\beta}^{\psi} - \mathbf{y}\|_2^2\right)\right]}.
\end{align*}
Thus the $\boldsymbol{\beta}^{a}$ from SMD will have
\begin{align}
    &\mathbbm{E} [\psi (\boldsymbol{\beta}^a) - \psi(\boldsymbol{\beta}^{\psi})]\nonumber\\
    &\quad\leq \frac{B}{s_m} \sqrt{\frac{\alpha}{(\alpha\eta-24L\eta^2)T} \left[2n D_{\psi}(\boldsymbol{\beta}^{\psi},\Tilde{\boldsymbol{\beta}}^{0})+ \frac{12L\eta^2 m}{\alpha}\left(\|X\Tilde{\boldsymbol{\beta}}^{0} - \mathbf{y}\|_2^2 -\|X\boldsymbol{\beta}^{\psi} - \mathbf{y}\|_2^2\right)\right]}\label{ch3:eq:19}\\
    &\mathbbm{E}[\|X\boldsymbol{\beta}^{a} - \mathbf{y}\|_2^2 - \|X\boldsymbol{\beta}^{\psi} - \mathbf{y}\|_2^2]\nonumber\\ &\quad\leq\frac{\alpha}{(\alpha\eta-24L\eta^2)T} \left[2n D_{\psi}(\boldsymbol{\beta}^{\psi},\Tilde{\boldsymbol{\beta}}^{0})+ \frac{12L\eta^2 m}{\alpha}\left(\|X\Tilde{\boldsymbol{\beta}}^{0} - \mathbf{y}\|_2^2 -\|X\boldsymbol{\beta}^{\psi} - \mathbf{y}\|_2^2\right)\right].\label{ch3:eq:19b}
\end{align}
Set the R.H.S. of \eqref{ch3:eq:19} and \eqref{ch3:eq:19b} to $\epsilon$ to get an $\epsilon-$solution, we can solve
\begin{align*}
    T &\geq \frac{B^2}{s_m^2\epsilon^2 }\frac{\alpha}{(\alpha\eta-24L\eta^2)} \left[2n D_{\psi}(\boldsymbol{\beta}^{\psi},\Tilde{\boldsymbol{\beta}}^{0})+ \frac{12L\eta^2 m}{\alpha}\left(\|X\Tilde{\boldsymbol{\beta}}^{0} - \mathbf{y}\|_2^2 -\|X\boldsymbol{\beta}^{\psi} - \mathbf{y}\|_2^2\right)\right]\\
    T &\geq \frac{1}{ \epsilon}\frac{\alpha}{(\alpha\eta-24L\eta^2)} \left[2n D_{\psi}(\boldsymbol{\beta}^{\psi},\Tilde{\boldsymbol{\beta}}^{0})+ \frac{12L\eta^2 m}{\alpha}\left(\|X\Tilde{\boldsymbol{\beta}}^{0} - \mathbf{y}\|_2^2 -\|X\boldsymbol{\beta}^{\psi} - \mathbf{y}\|_2^2\right)\right].
\end{align*}
That is, $T\sim \mathcal{O}(\frac{1}{\epsilon} + \frac{1}{\epsilon^2})$.

For part (b) of Theorem \ref{CH3:THM:03}, we first need to show that $F(\cdot)$ satisfies PL inequality w.r.t. $\|\cdot\|_2$. 
Remember $F(\boldsymbol{\beta}) = \frac{1}{2n}\|X\boldsymbol{\beta} - y\|_2^2$, then 
\begin{align*}
        \frac{1}{2}\|\nabla F(\boldsymbol{\beta})\|_2^2 &= \frac{1}{2}\|\frac{1}{n} X^T (X\boldsymbol{\beta}-y)\|_2^2\\
        &=\frac{1}{2n^2}\|X^T (X\boldsymbol{\beta}-P_{\mathrm{col}(X)}y)\|_2^2\\
        &\stackrel{\eqref{ch3:eq:lem4}}{\geq}\frac{s_m^2}{2n^2}\|X\boldsymbol{\beta}-P_{\mathrm{col}(X)}y\|_2^2\\
        &= \frac{s_m^2}{2n^2}\|X\boldsymbol{\beta}-y + P_{\mathcal{N}(X^T)}y\|_2^2\\
        &= \frac{s_m^2}{2n^2}[\|X\boldsymbol{\beta}-y\|_2^2 - \|P_{\mathcal{N}(X^T)}y\|_2^2]\\
        &= \frac{s_m^2}{n}[F(\boldsymbol{\beta}) - F(\boldsymbol{\beta}^{\psi})].
\end{align*}
Thus $F(\cdot)$ satisfies PL inequality with constant $\frac{s_m^2}{n}$. By Lemma \ref{ch3:lem:PL_QG}, $F(\cdot)$ satisfies QG with the same constant, so part (b) of Theorem 1 applies here, and we get
\begin{equation*}
        \mathbbm{E}[F(\boldsymbol{\beta}_a) - F(\boldsymbol{\beta}^{\psi})] \leq (\tau')^S [F(\Tilde{\boldsymbol{\beta}}^0) - F(\boldsymbol{\beta}^{\psi})].
\end{equation*}
where 
\begin{equation*}
\tau' = \frac{ \frac{12L\eta^2 }{\alpha} + \frac{\ell n}{m s_m^2}}{\eta - \frac{12L\eta^2}{\alpha}},
\end{equation*}
By the similar statement as the first part of the proof, we have
\begin{align*}
     &\mathbbm{E} [\psi (\boldsymbol{\beta}^a) - \psi(\boldsymbol{\beta}^{\psi})]\\
     \leq&\frac{B}{s_m}[2n\mathbbm{E}[F(\boldsymbol{\beta}^a) - F(\boldsymbol{\beta}^{\psi})]]^{.5}\\
     \leq & \frac{B \sqrt{2n} (\tau')^{S/2} }{s_m}[F(\Tilde{\boldsymbol{\beta}}^0) - F(\boldsymbol{\beta}^{\psi})]^{.5}.
\end{align*}
Thus 
\begin{align*}
    \mathbbm{E} [\psi (\boldsymbol{\beta}^a) - \psi(\boldsymbol{\beta}^{\psi})]&\leq \frac{B(\tau')^{S/2}}{s_m} \sqrt{\|X\Tilde{\boldsymbol{\beta}}^{0} - \mathbf{y}\|_2^2 -\|X\boldsymbol{\beta}^{\psi} - \mathbf{y}\|_2^2}\\
    \mathbbm{E}[\|X\boldsymbol{\beta}^{a} - \mathbf{y}\|_2^2 - \|X\boldsymbol{\beta}^{\psi} - \mathbf{y}\|_2^2] &\leq (\tau')^S\left(\|X\Tilde{\boldsymbol{\beta}}^{0} - \mathbf{y}\|_2^2 -\|X\boldsymbol{\beta}^{\psi} - \mathbf{y}\|_2^2\right).
\end{align*}

\end{proof}

\section{Proof of Corollary \ref{CH3:COR:01}}\label{ch3:app:cor1}
\begin{proof}
For the specific choice of $\psi(\cdot) = \frac{1}{2}\|\cdot\|_2^2$, we have
\begin{equation*}
    \nabla^2 \psi(\cdot) = I.
\end{equation*}
That is, $\psi(\cdot)$ is $1-$smooth and $1-$strongly convex, for 
\begin{equation*}
\tau'' = \frac{12L\eta^2 + \frac{n}{m s_m^2}}{\eta - 12L\eta^2}<1,    
\end{equation*}
we can apply part (b) of Theorem 2 to get
\begin{align*}
    \mathbbm{E}[\|X\boldsymbol{\beta}^{a} - X\boldsymbol{\beta}^{\psi}\|_2^2 &\leq (\tau'')^S\left(\|X\Tilde{\boldsymbol{\beta}}^{0} - \mathbf{y}\|_2^2 -\|X\boldsymbol{\beta}^{\psi} - \mathbf{y}\|_2^2\right)\\
    &= (\tau'')^S\left(\|\mathbf{y}\|_2^2 -\|P_{\mathcal{N}(X^T)}\mathbf{y}\|_2^2\right) = (\tau'')^S\|P_{\mathrm{col}(X)}\mathbf{y}\|_2^2.
\end{align*}
Then since $\nabla\psi(\boldsymbol{\beta}^a) = \boldsymbol{\beta}^a \in \mathrm{col}(X^T)$ and $\boldsymbol{\beta}^{\psi} = (X^T X)^{+} X^T \mathbf{y} = X^T (X X^T)^+\mathbf{y} \in  \mathrm{col}(X^T)$, we have
\begin{equation*}
\mathbbm{E}\|\boldsymbol{\beta}_a -\boldsymbol{\beta}^{\psi}\|_2^2\stackrel{\eqref{ch3:eq:lem4}}{\leq} \frac{1}{s_m^2}\mathbbm{E}\|X\boldsymbol{\beta}^{a} - X\boldsymbol{\beta}^{\psi}\|_2^2\leq  \frac{\tau^S}{s_m^2}\|P_{\mathrm{col}(X)}\mathbf{y}\|_2^2.
\end{equation*}
\end{proof}

\section{Proof of Theorem \ref{CH3:COR:02}}\label{ch3:app:F}

To prove Theorem \ref{CH3:COR:02}, we will need a few preliminaries.

\begin{lemma}\label{ch3:lem:08}
Denote $\boldsymbol{\beta}^{(0)} = \argmin_{\boldsymbol{\beta}}\{ \|\boldsymbol{\beta}\|_{1}:  X\boldsymbol{\beta} = P_{\mathrm{col}(X)} \mathbf{y}\}$ and $\boldsymbol{\beta}^{(\delta)} = \argmin_{\boldsymbol{\beta}}\{ \|\boldsymbol{\beta}\|_{1+\delta}^{1+\delta}:  X\boldsymbol{\beta} = P_{\mathrm{col}(X)} \mathbf{y}\}$, we have
\begin{equation}
    \|\boldsymbol{\beta}^{(\delta)}\|_1\leq p^{\frac{\delta}{1+\delta}}\|\boldsymbol{\beta}^{(0)}\|_1
\end{equation}
\end{lemma}
\begin{proof}
\begin{align*}
    \|\boldsymbol{\beta}^{(\delta)}\|_1 &= \sum_{i=1}^{p} 1*|\beta_i^{(\delta)}|\\
    &\leq \left(\sum_{i=1}^p 1^{\frac{1+\delta}{\delta}}\right)^{\frac{\delta}{1+\delta}}*\left(\sum_{i=1}^p |\beta_i^{(\delta)}|^{1+\delta}\right)^{\frac{1}{1+\delta}}\\
    &=p^{\frac{\delta}{1+\delta}} \|\boldsymbol{\beta}^{(\delta)}\|_{1+\delta}\\
    &\leq p^{\frac{\delta}{1+\delta}} \|\boldsymbol{\beta}^{(0)}\|_{1+\delta}\\
    &\leq p^{\frac{\delta}{1+\delta}} \left(\sum_{i=1}^p\|\boldsymbol{\beta}^{(0)}\circ \mathbf{e}_i\|_{1+\delta}\right)\\
    &=p^{\frac{\delta}{1+\delta}} \left(\sum_{i=1}^p|\beta^{(0)}_i|\right)\\
    &= p^{\frac{\delta}{1+\delta}}\|\boldsymbol{\beta}^{(0)}\|_1
\end{align*}
where the first inequality holds by Hölder's inequality, the second inequality follows from $\boldsymbol{\beta}^{(\delta)}$ minimizes $\|\boldsymbol{\beta}\|_{1+\delta}^{1+\delta}$ thus minimizes $\|\boldsymbol{\beta}\|_{1+\delta}$ among all $\boldsymbol{\beta}\in \mathcal{B}:= \{\boldsymbol{\beta}:X\boldsymbol{\beta} = P_{\mathrm{col(X)}}\mathbf{y}\}$, and the third inequality follows from triangle inequality.
\end{proof}

\begin{lemma}\label{ch3:lem:10}
Assume $\mathbf{y} = X\boldsymbol{\beta}^o$ where $\boldsymbol{\beta}^o\leq s$ and $X$ is $s-$good with parameter $\kappa<1/2$. Use the notation in Lemma \ref{ch3:lem:08} that $\boldsymbol{\beta}^{(0)} = \argmin_{\boldsymbol{\beta}}\{ \|\boldsymbol{\beta}\|_{1}:  X\boldsymbol{\beta} = \mathbf{y}\}$, we have $\boldsymbol{\beta}^{(0)} = \boldsymbol{\beta}^o$.
\end{lemma}
\begin{proof}
$ \boldsymbol{\beta}^o$ is a feasible solution to $\min_{\boldsymbol{\beta}}\{ \|\boldsymbol{\beta}\|_{1}:  X\boldsymbol{\beta} = \mathbf{y}\}$. Denote $\hat{\mathbf{u}} = \boldsymbol{\beta}^o - \boldsymbol{\beta}^{(0)}$, we have
\begin{eqnarray}
\label{ch3:eq:lem10eq1}
        & \left(\|\boldsymbol{\beta}^{(0)}_S\|_1 + \|\boldsymbol{\beta}^{(0)}_{S^c}\|_1\right) = \|\boldsymbol{\beta}^{(0)}\|_1 \leq  \|\boldsymbol{\beta}^{o}\|_1 =\|\boldsymbol{\beta}^{o}_S\|_1 \cr
        \Longrightarrow&\|\boldsymbol{\beta}^{(0)}_{S^c}\|_1\leq \|\boldsymbol{\beta}^{o}_S\|_1 - \|\boldsymbol{\beta}^{(0)}_S\|_1 \leq \|\boldsymbol{\beta}^{o}_S - \boldsymbol{\beta}^{(0)}_S\|_1\cr
        \Longleftrightarrow & \|\hat{\mathbf{u}}_{S^c}\|_1\leq \|\hat{\mathbf{u}}_S\|_1.
\end{eqnarray}

On the other hand, 
\begin{align*}
    X \hat{\mathbf{u}} = X\boldsymbol{\beta}^{o} - X\boldsymbol{\beta}^{(0)} = \mathbf{y} - \mathbf{y} = \mathbf{0}.
\end{align*}
That is, $\hat{\mathbf{u}}\in \mathcal{N}(X)$, then by $X$ is $s-$good, we have for $\kappa < \frac{1}{2}$
\begin{align*}\label{ch3:eq:lem10eq3}
        &\|\hat{\mathbf{u}}_S\|_1\leq \kappa\|\hat{\mathbf{u}}\|_1\\
        \Longrightarrow&\|\hat{\mathbf{u}}_S\|_1\leq \frac{\kappa}{1 - \kappa}\|\hat{\mathbf{u}}_{S^c}\|_1.\numberthis
\end{align*}
Combine \eqref{ch3:eq:lem10eq1} and \eqref{ch3:eq:lem10eq3} we get
\begin{eqnarray*}
    &\|\hat{\mathbf{u}}_{S^c}\|_1\leq\frac{\kappa}{1 - \kappa}\|\hat{\mathbf{u}}_{S^c}\|_1 \cr
    \Longrightarrow &\|\hat{\mathbf{u}}_{S^c}\|_1 = 0\cr
    \stackrel{\eqref{ch3:eq:lem10eq3}}{\Longrightarrow}& \|\hat{\mathbf{u}}_{S}\|_1 = 0.
\end{eqnarray*}
Thus $\hat{\mathbf{u}} = \mathbf{0}$, $\boldsymbol{\beta}^{(0)} = \boldsymbol{\beta}^o$. 
\end{proof}

\begin{lemma}\label{ch3:lem:09}
Assume $X\in \mathbbm{R}^{n\times p}$ is $s-good$ with parameter $\kappa<1/2$, then $\forall \boldsymbol{\beta} \in \mathbbm{R}^p$ and $I\subset\{1,\ldots,p\}$ with $|I| \leq s$ we have
\begin{equation*}
    \|\boldsymbol{\beta}_I\|_1\leq \frac{\sqrt{s} +\kappa\sqrt{p}}{s_m}\|X\boldsymbol{\beta}\|_2 + \kappa \|\boldsymbol{\beta}\|_1
\end{equation*}
where $s_m^2$ is the smallest nonzero eigenvalue of $X^T X$.
\end{lemma}
\begin{proof}
\begin{align*}
        \|\boldsymbol{\beta}_I\|_1 &\leq \|(P_{\mathrm{col}(X^T)}\boldsymbol{\beta})_I\|_1 + \|(P_{\mathcal{N}(X)}\boldsymbol{\beta})_I\|_1\\
        &\leq \sqrt{s}\|(P_{\mathrm{col}(X^T)}\boldsymbol{\beta})_I\|_2 + \kappa\|P_{\mathcal{N}(X)}\boldsymbol{\beta}\|_1\\
        &\leq \sqrt{s}\|P_{\mathrm{col}(X^T)}\boldsymbol{\beta}\|_2 + \kappa(\|P_{\mathcal{N}(X)}\boldsymbol{\beta} + P_{\mathrm{col}(X^T)}\boldsymbol{\beta}\|_1 +\|P_{\mathrm{col}(X^T)}\boldsymbol{\beta}\|_1 )\\
        &\stackrel{\eqref{ch3:eq:lem4}}{\leq}\frac{\sqrt{s}}{s_m}\|X P_{\mathrm{col}(X^T)}\boldsymbol{\beta}\|_2 + \kappa\left(\|\boldsymbol{\beta}\|_1 +\frac{\sqrt{p}}{s_m}\|X P_{\mathrm{col}(X^T)}\boldsymbol{\beta}\|_2 \right)\\
        &=\frac{\sqrt{s} +\kappa\sqrt{p}}{s_m}\|X\boldsymbol{\beta}\|_2 + \kappa \|\boldsymbol{\beta}\|_1
\end{align*}
where the first inequality holds by triangle inequality, the second inequality holds by Cauchy-Schwartz inequality and s-goodness of $X$, the third inequality again uses the triangle inequality.
\end{proof}

\begin{theorem}\label{ch3:thm:arkadi}[Theorem 1.3.1 in \cite{ben2011lectures}] Consider the following optimization problem with any norm $\|\cdot\|$
\begin{align}
\begin{split}\label{ch3:eq:opt:arkadi}
    \hat{\boldsymbol{\beta}} = \arg\min_{\boldsymbol{\beta}}&\quad \|\boldsymbol{\beta}\|_1\\
     s.t. &\quad\|X \boldsymbol{\beta} - \mathbf{y}\|\leq \Delta
\end{split}
\end{align}
Assume $\exists \alpha\geq 0$ and $\kappa'<1/2$ that $\forall \boldsymbol{\beta} \in \mathbbm{R}^p$ and $\forall I\subset\{1,\ldots,p\}$ with $|I| \leq s$:
\begin{equation*}
    \|\boldsymbol{\beta}_I\|_1\leq \alpha\|X\boldsymbol{\beta}\| + \kappa' \|\boldsymbol{\beta}\|_1
\end{equation*}
Suppose $\Tilde{\boldsymbol{\beta}}$ is an approximate solution to \eqref{ch3:eq:opt:arkadi}  such that
\begin{align*}
        \|\Tilde{\boldsymbol{\beta}}\|_1 &\leq \|\hat{\boldsymbol{\beta}}\|_1 + \nu\\
        \|X \Tilde{\boldsymbol{\beta}} - \mathbf{y}\|&\leq \Delta + \epsilon.
\end{align*}
Then for any nearly $s-$sparse and feasible $\boldsymbol{\beta}'$ for \eqref{ch3:eq:opt:arkadi}, where near $s-$sparsity means $\exists$ $s-$sparse $\boldsymbol{\beta^s}$ such that $\|\boldsymbol{\beta}' - \boldsymbol{\beta^s}\|_1\leq v$, we have
\begin{equation*}
    \|\Tilde{\boldsymbol{\beta}} - \boldsymbol{\beta}'\|_1\leq \frac{2\alpha(2\Delta + \epsilon) + 2v + \nu}{1 - 2\kappa'}.
\end{equation*}
\end{theorem}
\vspace{10pt}
We are ready to prove Theorem \ref{CH3:COR:02}.

\begin{proof}[Proof for Theorem \ref{CH3:COR:02}]
By Lemma \ref{ch3:lem:08} we have 
\begin{align*}
    \|\boldsymbol{\beta}^{(\delta)}\|_1 \leq \|\boldsymbol{\beta}^{(0)}\|_1 + \left(p^{\frac{\delta}{1+\delta}}-1\right)\|\boldsymbol{\beta}^{(0)}\|_1
\end{align*}
And by Lemma \ref{ch3:lem:09}
\begin{equation*}
    \|\boldsymbol{\beta}_I\|_1\leq \frac{\sqrt{s} +\kappa\sqrt{p}}{s_m}\|X\boldsymbol{\beta}\|_2 + \kappa \|\boldsymbol{\beta}\|_1
\end{equation*}
Thus we can apply Theorem \ref{ch3:thm:arkadi} by setting $\Delta = 0$, and $\|\cdot\| = \|\cdot\|_2$, then $\boldsymbol{\hat{\beta}} = \boldsymbol{\beta}^{(0)}$. Take $\Tilde{\boldsymbol{\beta}} = \boldsymbol{\beta}^{(\delta)}$, $\boldsymbol{\beta}' = \boldsymbol{\beta}^o$, then $\nu = \left(p^{\frac{\delta}{1+\delta}}-1\right)\|\boldsymbol{\beta}^{(0)}\|_1$, $\epsilon = 0$ , $v = 0 $, $\alpha = \frac{\sqrt{s} +\kappa\sqrt{p}}{s_m}$, $\kappa' = \kappa$, then we have
\begin{equation}\label{ch3:eq:cor2eq3}
    \|\boldsymbol{\beta}^{(\delta)} - \boldsymbol{\beta}^o\|_1\leq \frac{ \left(p^{\frac{\delta}{1+\delta}}-1\right)\|\boldsymbol{\beta}^{(0)}\|_1}{1 - 2\kappa} = \frac{ \left(p^{\frac{\delta}{1+\delta}}-1\right)\|\boldsymbol{\beta}^o\|_1}{1 - 2\kappa}
\end{equation}
where the equality is from Lemma \ref{ch3:lem:10}.

We can further bound $\|\boldsymbol{\beta}^o\|_1$. By $(s,\gamma)-RE$ condition we have
\begin{align*}
    \|\boldsymbol{\beta}^o\|^2_2\leq \frac{\|X \boldsymbol{\beta}^o\|^2_2}{n\gamma} = \frac{\|\mathbf{y}\|_2^2}{n\gamma}.
\end{align*}
Thus 
\begin{align}\label{ch3:eq:cor2eq5}
    \|\boldsymbol{\beta}^o\|_1\leq \sqrt{s}\|\boldsymbol{\beta}^o\|_2\leq \sqrt{\frac{s}{n\gamma}}\|\boldsymbol{y}\|_2.
\end{align}
Plug \eqref{ch3:eq:cor2eq5} into \eqref{ch3:eq:cor2eq3}, we have
\begin{align*}
   \|\boldsymbol{\beta}^{(\delta)} - \boldsymbol{\beta}^o\|_1 &\leq \frac{ \left(p^{\frac{\delta}{1+\delta}}-1\right)}{1 - 2\kappa}\sqrt{\frac{s}{n\gamma}}\|\boldsymbol{y}\|_2\\
   &\leq \frac{ \left(p^{\frac{\log \left(1+\frac{(1-2\kappa)\sqrt{n\gamma}}{\sqrt{s}\|\mathbf{y}\|_2}\xi\right)}{\log p}}-1\right)}{1 - 2\kappa} \sqrt{\frac{s}{n\gamma}}\|\boldsymbol{y}\|_2\\
   &= \frac{ \left(1 + \frac{(1-2\kappa)\sqrt{n\gamma}}{\sqrt{s}\|\mathbf{y}\|_2}\xi-1\right)}{1 - 2\kappa}\sqrt{\frac{s}{n\gamma}}\|\boldsymbol{y}\|_2 =  \xi.
\end{align*}

\end{proof}

\end{document}